\documentclass[twoside,11pt]{article}

\usepackage{jmlr2e}
\usepackage{amsmath}
\usepackage{multirow}
\usepackage{amsfonts}
\usepackage{booktabs}
\usepackage{algorithm}
\usepackage[noend]{algorithmic}
\usepackage{mathrsfs}
\hypersetup{colorlinks=true,citecolor=blue}
\numberwithin{theorem}{section}
\usepackage{adjustbox}
\usepackage[most]{tcolorbox}
\usepackage{xcolor}
\usepackage{wrapfig}

\newcommand{\br}{\mathbb{R}}

\newcommand{\EE}{\mathbb{E}}
\newcommand{\PP}{\mathbb{P}}
\newcommand{\x}{\mathbf{x}}
\newcommand{\y}{\mathbf{y}}
\newcommand{\z}{\mathbf{z}}
\newcommand{\g}{\mathbf{g}}
\newcommand{\ACal}{\mathcal{A}}

\newcommand{\CCal}{\mathcal{C}}
\newcommand{\FCal}{\mathcal{F}}
\newcommand{\NCal}{\mathcal{N}}
\newcommand{\RCal}{\mathcal{R}}
\newcommand{\XCal}{\mathcal{X}}
\newcommand{\VCal}{\mathcal{V}}
\newcommand{\LCal}{\mathcal{L}}

\newcommand{\argmax}{\mathop{\rm argmax}}
\newcommand{\sign}{\textnormal{sign}}

\usepackage{lastpage}
\jmlrheading{XXX}{2026}{XX-XX}{XX}{XX}{XXX}{Peter Chen and Xi Chen and Wotao Yin and Tianyi Lin}

\ShortHeadings{Comparison-Based Preference Alignment}{Chen, Chen, Yin, and Lin}
\firstpageno{1}

\begin{document}

\title{A Zeroth-Order Paradigm for LLM Preference Alignment}

\author{\name Peter Chen \email pllc@eecs.berkeley.edu \\
       \addr Department of Electrical Engineering and Computer Sciences (EECS) \\
       University of California, Berkeley \\
       Berkeley, CA 94720, USA
       \AND
       \name Xi Chen \email xc13@stern.nyu.edu \\
       \addr Stern School of Business\\
       New York University\\
       New York, NY 10012, USA
       \AND
       \name Wotao Yin \email
       wotao.yin@alibaba-inc.com \\
       \addr Decision Intelligence Lab (Seattle) \\
       DAMO Academy, Alibaba Group U.S. \\
       Bellevue, WA 98004, USA
       \AND
       \name Tianyi Lin \email
       tl3335@columbia.edu \\
       \addr Department of Industrial Engineering and Operations Research (IEOR) \\
       Columbia University \\
       New York, NY 10027, USA
       }

\maketitle
\vspace{-1.75em}
\begin{abstract}
Direct preference alignment methods are widely used to align large language models (LLMs) with human preferences because of their computational and memory efficiency. However, likelihood displacement motivates alternative ways to extract information from preference pairs with small likelihood margins. In this paper, we propose and analyze Comparison-based Preference Optimization (ComPO), a zeroth-order alignment method based on comparison oracles. ComPO extracts directional information from these pairs without directly optimizing a differentiable preference loss on them. We establish a convergence guarantee for its basic offline scheme under smoothness, gradient sparsity, and compatibility between the oracle and a latent objective. We further introduce online ComPO, which retains the offline comparison mechanism and uses unlabeled policy generations for reverse-KL control relative to a reference policy. Following the coverage perspective of preference fine-tuning, we establish a performance guarantee for a basic constrained scheme under local coverage and in-distribution pairwise reward accuracy. Experiments on Mistral, Llama, Gemma-2, Qwen3, and Gemma-3 models demonstrate improvements over existing direct alignment methods, including length-controlled win rates, with pair-level diagnostics providing
evidence consistent with mitigating likelihood displacement.
\end{abstract}


\section{Introduction}\label{sec:intro}
Generative AI has become an increasingly important tool for building and managing intelligent systems across academia, industry, and government. Large language models (LLMs) are a core part of this progress, with strong capabilities in data organization, retrieval, reasoning, and analysis~\citep{Brown-2020-Language, Chowdhery-2023-Palm, Touvron-2023-Llama, Achiam-2023-GPT, Bubeck-2023-Sparks}. Since these models are trained on large and heterogeneous corpora, they need further alignment with human preferences so that their responses are helpful, harmless, and reliable~\citep{Bai-2022-Training}.
A prominent approach is \textit{reinforcement learning from human feedback} (RLHF)~\citep{Christiano-2017-Deep, Stiennon-2020-Learning}, which first learns a reward model from human preference pairs and then optimizes the policy using reinforcement learning. Despite its empirical success~\citep{Ziegler-2019-Fine, Ouyang-2022-Training, Touvron-2023-Llama, Achiam-2023-GPT}, RLHF requires a multi-stage training pipeline and can be expensive in memory and computation. This motivates direct alignment methods, e.g., direct preference optimization (DPO)~\citep{Rafailov-2023-Direct} and its variants~\citep{Azar-2024-General, Ethayarajh-2024-Model, Park-2024-Disentangling, Xu-2024-Contrastive, Tang-2024-Generalized, Meng-2024-SimPO, Chen-2025-MallowsPO, Zhao-2025-RainbowPO}, which directly optimize the policy using preference pairs and avoid separately training a reward model.

Direct alignment methods are appealing because of their simplicity and stability. Yet, they suffer from a critical issue known as likelihood displacement. \emph{Likelihood displacement} refers to the counter-intuitive situation where training increases the likelihood of preferred responses relative to dispreferred ones, but decreases the absolute probability of the preferred responses, leading to ``unintentional unalignment''~\citep{Pal-2024-Smaug, Tajwar-2024-Preference, Rafailov-2024-From,
Pang-2024-Iterative, Liu-2024-Provably, Yuan-2025-Advancing, Razin-2025-Unintentional}. For example, training a model to prefer \textsc{No} over \textsc{Never} can sharply increase the likelihood of \textsc{Yes}. Practically, this issue can harm LLM behavior by shifting probability mass to unsafe responses. When the prompt asks for steps for a terrorist organization to infiltrate a government agency,
Gemma-2B-it initially generates refusal responses, while DPO training can make the model comply with the unsafe request because likelihood displacement shifts probability mass away from refusal responses; see~\citet[Table~18]{Razin-2025-Unintentional}. Another related issue is \emph{verbosity}, which refers to the tendency of models fine-tuned with RLHF~\citep{Singhal-2024-Long, Kabir-2024-Stack} or direct alignment methods~\citep{Park-2024-Disentangling, Amini-2024-Direct, Rafailov-2024-Scaling} to generate longer responses without a corresponding improvement in quality, resulting in lower efficiency and higher consumption of hardware resources.

Recent works have suggested that likelihood displacement is related to preference pairs whose preferred and dispreferred responses are similar under model-dependent measures~\citep{Pal-2024-Smaug, Razin-2025-Unintentional}. We refer to such small margin pairs as \textit{noisy preference pairs} in this paper (see Eq.~\eqref{eq:noisy-margin}). Existing methods have tried to mitigate likelihood displacement by adding additional regularization~\citep{Pal-2024-Smaug, Rafailov-2024-From}. More recently,~\citet{Razin-2025-Unintentional} proposed to measure the similarity between preferred and dispreferred responses using the centered hidden embedding similarity (CHES) score, and empirically showed that filtering out preference pairs identified by the CHES score as problematic can be more effective for mitigating likelihood displacement than adding supervised fine-tuning (SFT) regularization. This finding highlights the role of data geometry in direct alignment. However, filtering noisy pairs also removes them from training entirely, even though these pairs may still contain useful comparative information.

While DPO provides a computationally convenient framework by maximizing a certain log-likelihood margin between preferred and dispreferred responses, this objective function can be viewed as a \emph{proxy} for the true goal of alignment. This proxy is effective when preference pairs clearly distinguish better responses from worse responses. However, when faced with noisy pairs -- \emph{where the preference signal is weak or ambiguous under model-based similarity measures} -- optimizing a fixed DPO-style objective can lead to adverse effects such as likelihood displacement. In such cases, the pair may still provide useful local information, even if it is not suitable for direct optimization by a margin-based loss. This motivates a comparison-oracle view of preference alignment. Explicitly defining alignment as a single optimizable mathematical objective function is exceptionally challenging. Instead of pursuing such an explicit objective, we ask whether a nearby policy perturbation improves the local behavior of the model on preference pairs. A favorable perturbation should increase the likelihood of the preferred response and decrease the likelihood of the dispreferred response. In this way, noisy preference pairs are treated as comparison signals about a latent alignment objective, rather than as direct samples for a fixed loss function.

In this paper, we propose a zeroth-order preference alignment method based on comparison oracles, called ComPO. Our approach perturbs the current policy, evaluates whether each perturbation increases the likelihood of preferred responses and decreases that of dispreferred responses, and aggregates the resulting one-bit signals to estimate a normalized update direction. This allows low-margin pairs, designated as noisy, to contribute to alignment without directly optimizing a differentiable preference loss on them, complementing standard direct alignment methods applied to clean pairs. We further extend ComPO to control policy deviation using unlabeled online generations, while retaining offline preference pairs as the source of comparison signals. The motivation follows the coverage perspective of~\citet{Song-2024-Importance}: reverse KL can be estimated from generations of the policy being evaluated, and constraining it permits a performance analysis based on the coverage within a prescribed neighborhood of the reference policy rather than over the full policy class. Coverage remains a separate assumption and is not implied by the KL constraint.

Our empirical also examines length-related effects which have been studied in the literature~\citep{Gao-2023-Scaling, Dubois-2023-Alpacafarm, Park-2024-Disentangling, Amini-2024-Direct, Xu-2024-Contrastive, Meng-2024-SimPO, Pang-2024-Iterative}. Although ComPO is not specifically designed to control verbosity, we evaluate its length-controlled (LC) win rates and examine pair-level likelihood changes. We interpret higher LC win rates as improved judged performance after adjustment for response length, rather than direct evidence of shorter responses.

\paragraph{Contributions.} Our contributions can be summarized as follows:
\begin{enumerate}
\item We develop ComPO, a comparison-based method that uses low-margin preference pairs to refine an aligned policy without directly optimizing a differentiable preference loss on those pairs. Its practical offline implementation uses output-layer perturbations and entry-wise thresholding. The online extension retains the same comparison mechanism and uses
unlabeled current-policy generations to adapt the step size.
\item We establish a best-iterate convergence guarantee for the basic offline scheme under smoothness, gradient sparsity, and oracle compatibility. For the basic online scheme, we prove feasibility under an exact reverse-KL constraint and bound the performance gap in terms of in-distribution pairwise reward error under local coverage.
\item We evaluate ComPO on base and instruction-tuned models from the Mistral, Llama, Gemma-2, Qwen3, and Gemma-3 families. The experiments assess its compatibility with direct alignment methods, its design choices, and the effects of online damping and replay. Pair-level likelihood diagnostics complement the benchmark evaluations.
\end{enumerate}

\paragraph{Relationship to the conference version.} A preliminary version of this work appeared at NeurIPS 2025~\citep{Chen-2025-Compo}. It introduced offline ComPO, preference comparison oracle, the convergence analysis, and the original offline experiments. The journal extension adds the online extension, its coverage-based analysis, and experiments on additional model families, including evaluations of online
regularization and replay.

\paragraph{Related works.} Direct preference alignment methods, including DPO~\citep{Rafailov-2023-Direct}, are simple and more stable offline alternatives to RLHF. Several DPO variants with alternative objectives have been proposed, including ranking-based variants beyond pairwise preference data~\citep{Dong-2023-RAFT, Yuan-2023-RRHF, Song-2024-Preference, Chen-2024-Noise, Liu-2025-LiPO} and reference-model-free variants~\citep{Hong-2024-ORPO, Meng-2024-SimPO}. It is well known that DPO suffers from the issues of verbosity~\citep{Park-2024-Disentangling, Amini-2024-Direct, Rafailov-2024-Scaling} and likelihood displacement~\citep{Pal-2024-Smaug, Tajwar-2024-Preference, Rafailov-2024-From, Pang-2024-Iterative, Liu-2024-Provably, Yuan-2025-Advancing}, which can be interpreted from a unified perspective of data curation~\citep{Park-2024-Disentangling, Razin-2025-Unintentional}. Our work continues along this perspective by arguing that these issues can be mitigated by using the information contained in noisy preference pairs for which the reference model assigns similar likelihoods to preferred and dispreferred responses.

Recent work has examined different roles of online data in preference fine-tuning. Online preference optimization can acquire additional labels for responses generated by the current policy, as in the online AI feedback approach of~\citet{Guo-2024-Direct}. In contrast,~\citet{Song-2024-Importance} introduce HyPO, which combines offline preference optimization with reverse-KL regularization estimated from unlabeled online samples. Our online extension follows this separation between preference supervision and regularization,
but uses comparison-derived update directions. A complementary line of work studies active exploration~\citep{Xie-2025-Exploratory} by augmenting online DPO with an explicit exploration bonus to guide the acquisition of preference feedback. Online ComPO does not acquire new preference labels or introduce such a bonus and its analysis concerns policy performance under local coverage.

Comparison-based optimization includes coordinate-search methods~\citep{Jamieson-2012-Query, Matsui-2017-Parallel} and directional estimators such as SCOBO~\citep{Cai-2022-One} and Sign-OPT~\citep{Cheng-2020-SignOPT}. Sign-OPT also provides a stationarity analysis under smoothness and additional assumptions on gradient noise, so nonconvexity alone is not the distinction from that work. ComPO specializes the comparison mechanism to preference alignment: its oracle evaluates preferred- and dispreferred-response likelihood changes, its basic analysis exploits approximately sparse gradients, and its practical implementation uses output-layer perturbations and thresholding. Comparison and ranking feedback have also been studied in bandit optimization~\citep{Yue-2009-Interactively,
Kumagai-2017-Regret, Ding-2018-Preference}, Bayesian optimization~\citep{Astudillo-2020-Multi, Lin-2022-Preference}, and RLHF~\citep{Tang-2024-Zeroth, Zhang-2025-Zeroth}.
Our focus is on extracting comparison signals from low-margin offline preference pairs and combining them with unlabeled online generations for step-size control.
\section{Preliminaries}\label{sec:prelim}
We provide an overview of the setup for direct preference alignment, and recall the definition of comparison oracles and the subroutine for estimating gradients using comparison oracles that are important for designing the basic scheme of our method. We further introduce the reverse-KL and coverage notation used in online ComPO.

\subsection{Direct preference alignment}
Modern LLMs are designed based on the Transformer architecture~\citep{Vaswani-2017-Attention}
and follow user prompts $\x \in \VCal^\star$ to generate responses $\y \in \VCal^\star$, where $\VCal$ is a vocabulary of tokens. We view an LLM as a policy $\pi_\theta(\y|\x)$ which assigns probabilities to responses $\y$ given prompts $\x$. To assign probabilities to each token of $\y$, the policy $\pi_\theta$ operates in an auto-regressive manner as follows,
\begin{equation*}
\pi_\theta(\y | \x) = \prod_{k=1}^{|\y|} \pi_\theta(\y_k | \x, \y_{<k}),
\end{equation*}
where $\theta$ denotes the model parameters (e.g., the parameters of the Transformer architecture) and $\y_{<k}$ denotes the first $k-1$ tokens of $\y$. However, the generations might not be helpful, safe, or reliable, which motivates further alignment of LLMs with human preferences.

We consider the direct preference learning pipeline based on pairwise preference data. Specifically, we assume access to a preference dataset $D$ containing samples $(\x,\y^+,\y^-)$, where $\x$ is a prompt and $(\y^+,\y^-)$ is a pair of preferred and dispreferred responses to $\x$. This pipeline usually includes an initial supervised fine-tuning (SFT) phase, where the model is fine-tuned using the cross-entropy loss and high-quality data for specific downstream tasks. The SFT data can be either independent
of $D$~\citep{Touvron-2023-Llama}, or may consist of prompts and preferred responses from $D$~\citep{Rafailov-2023-Direct}.

Direct alignment methods, such as DPO~\citep{Rafailov-2023-Direct}, optimize the policy $\pi_\theta$ over the preference dataset $D$ without learning a reward model as in RLHF~\citep{Ziegler-2019-Fine, Stiennon-2020-Learning}. This is done by minimizing a contrastive loss as follows, 
\begin{equation}\label{eq:DPO}
\LCal_{\textnormal{DPO}}(\theta) = -\EE_{(\x, \y^+, \y^-) \sim D}\left[\log\sigma\left(\beta\log\tfrac{\pi_\theta(\y^+ | \x)}{\pi_{\textnormal{ref}}(\y^+ | \x)}-\beta\log\tfrac{\pi_\theta(\y^- | \x)}{\pi_{\textnormal{ref}}(\y^- | \x)}\right)\right],
\end{equation}
where $\pi_{\textnormal{ref}}$ is the model after SFT, $\beta$ is a regularization parameter, and $\sigma: \br \rightarrow [0,1]$ is the sigmoid function. The function $\LCal_{\textnormal{DPO}}$ relies on the log-likelihood margin between $\y^+$ and $\y^-$. Thus, DPO improves the relative likelihood margin between the two responses, rather than directly maximizing the likelihood of $\y^+$ and minimizing the likelihood of $\y^-$. During training, the likelihood of $\y^+$ might decrease, and probability mass can be shifted from $\y^+$ to responses with an opposite meaning~\citep{Pal-2024-Smaug, Razin-2025-Unintentional}. A possible reason is that the above objective function is not well suited for extracting information from
noisy preference pairs whose preferred and dispreferred responses have small likelihood margins or are similar under model-based measures.

Empirically, \citet{Razin-2025-Unintentional} show that filtering out similar preference pairs can make DPO more effective. However, noisy preference pairs might still contain useful information that can improve the performance of LLMs. Extracting such information is challenging using a fixed margin-based loss, since maximizing the likelihood of $\y^+$ and minimizing the likelihood of $\y^-$ locally does not by itself define a global alignment objective. The local information we use is comparative: a better policy should assign higher likelihood to $\y^+$ and lower likelihood to $\y^-$. This motivates us to design a new alignment method by directly leveraging the comparison signal in pairwise preference data $(\x,\y^+,\y^-)$ from $D$.

\subsection{Comparison oracles and zeroth-order methods}
To contextualize our proposed method for aligning LLMs with human preferences, we review the definition of comparison oracles and explain how comparison oracles can be used to develop zeroth-order methods.

Given a function $f:\br^d \to \br$ for which neither the function value nor the gradient is accessible, we define a pairwise comparison oracle $\CCal_f$ in its simplest form as follows, 
\begin{definition}
We call $\CCal_f(\theta,\theta'):\br^d \times \br^d \to \{+1,-1\}$ a comparison oracle for
function $f$ if
\begin{equation*}
\CCal_f(\theta, \theta') = \left\{
\begin{array}{cl}
-1, & \textnormal{if } f(\theta') < f(\theta), \\
+1, & \textnormal{otherwise}.
\end{array} \right.
\end{equation*}
In other words, when queried with $\theta$ and $\theta'$, the oracle $C_f(\cdot, \cdot)$ returns $-1$ if $f(\theta')<f(\theta)$ and $+1$ otherwise, with ties assigned to $+1$.
\end{definition}

The key idea behind the subroutine in~\citet{Cai-2022-One} for estimating gradients using comparison oracles is inspired by $1$-bit compressed sensing~\citep{Boufounos-2008-Compressed}. The goal is to recover a signal $\g \in \br^d$ from quantized measurements $y_i=\sign(\z_i^\top \g)$, where $\z_i$ is a random perturbation vector drawn from a rotationally invariant distribution. The theoretical guarantee on the required number of perturbations to obtain an approximate signal was established in~\citet{Plan-2012-Robust} and extended in~\citet{Cai-2022-One}. Notably, for a small perturbation radius $r>0$, we have
\begin{equation*}
\CCal_f(\theta, \theta+r\z_i) = \sign(f(\theta+r\z_i)-f(\theta)) \approx \sign(\z_i^\top \nabla f(\theta)).
\end{equation*}
Here, $\sign(0)=+1$. Thus, the comparison label $y_i=\CCal_f(\theta,\theta+r\z_i)$ serves as an approximate one-bit measurement of $\nabla f(\theta)$.

Another issue is that zeroth-order comparison-based methods can suffer from dimension-dependent iteration complexity bounds~\citep{Jamieson-2012-Query}. This is expected because comparison oracles are even weaker than function-value oracles. This dimension dependence can be mitigated by exploiting sparse gradient structure~\citep{Wang-2018-Stochastic,Golovin-2020-Gradientless,Choromanski-2019-Complexity, Cai-2022-One,Cai-2022-Zeroth}. Indeed, we say that the function $f$ has sparse gradients if $\|\nabla f(\theta)\|_1 \leq \sqrt{s}\|\nabla f(\theta)\|$ for all $\theta\in\br^d$ and some $s\ll d$.

The above discussion gives the subroutine for estimating sparse gradients using comparison oracles. We generate $m$ i.i.d. perturbation vectors, denoted by $\{\z_i\}_{1\leq i\leq m}$, compute $y_i=\CCal_f(\theta,\theta+r\z_i)$ for all $i$, and solve the following optimization problem:
\begin{equation}\label{eq:1BGE}
\hat{\g} = \argmax_{\|\g\|_1 \leq \sqrt{s}, \|\g\| \leq 1} \sum_{i=1}^m y_i \z_i^\top \g,
\end{equation}
where the constraints $\|\g\|_1 \leq \sqrt{s}$ and $\|\g\| \leq 1$ restrict the search to an approximately sparse and normalized set.

In ComPO, the latent function $f$ is viewed as an implicit alignment objective. Instead of assuming access to its function value or gradient, we use offline preference pairs to construct a comparison oracle: a nearby policy is considered better if it assigns a higher likelihood to the preferred response and a lower likelihood to the dispreferred response.

\subsection{Reverse KL and local coverage}\label{subsec:rkl-coverage}
The online extension of ComPO uses unlabeled policy generations for regularization, while the comparison oracle continues to use the fixed offline preference pairs. Let $P_\textnormal{on}$ denote the prompt distribution used for online generation, and let
$\pi_\textnormal{ref}$ be a fixed reference policy. For the online analysis, we consider policies with a common response support and positive probabilities on that support, and assume that the relevant expectations are finite. For any such policy $\pi$, we define its sequence-level reverse KL relative to the reference by
\begin{equation}\label{def:rkl}
D_\textnormal{RKL}(\pi\|\pi_\textnormal{ref}) = \EE_{\x\sim P_\textnormal{on}}[D_\textnormal{KL}(\pi(\cdot|\x)\|\pi_\textnormal{ref}(\cdot|\x))] = \EE_{\x\sim P_\textnormal{on},\y\sim\pi(\cdot|\x)}\left[\log\tfrac{\pi(\y|\x)}{\pi_\textnormal{ref}(\y|\x)}\right], 
\end{equation}
The reverse KL can be estimated using unlabeled generations from the current policy being evaluated. For $\tau>0$, we define the reverse-KL neighborhood of the reference policy by 
\begin{equation}\label{eq:kl-ball}
\Pi_\tau=\{\pi:D_\textnormal{RKL}(\pi\|\pi_\textnormal{ref})\leq\tau\}
\end{equation}
We let $r^\star(\x,\y)$ denote the ground-truth reward. For $\beta>0$, we define the KL-regularized population objective by
\begin{equation}\label{def:obj-kl}
J_\beta(\pi) = \EE_{\x\sim P_\textnormal{on},\y\sim\pi(\cdot|\x)}[r^\star(\x,\y)]-\beta D_\textnormal{RKL}(\pi\|\pi_\textnormal{ref}).
\end{equation}
For a policy $\pi$, we define its implicit reward relative to $\pi_\textnormal{ref}$ by
\begin{equation}\label{def:implicit-reward}
\widehat{r}_\pi(\x,\y)=\beta\log\tfrac{\pi(\y|\x)}{\pi_\textnormal{ref}(\y|\x)}.
\end{equation}
Pairwise reward differences are invariant to prompt-dependent additive constants. We thus measure the accuracy through the following in-distribution pairwise error, where $\y_1$ and $\y_2$ are drawn independently from $\pi_\textnormal{ref}(\cdot|\x)$ conditional on $\x$. Formally, we have
\begin{equation}\label{def:error}
\operatorname{err}(\pi) = \EE_{\x\sim P_\textnormal{on},\y_1,\y_2\sim\pi_\textnormal{ref}(\cdot|\x)}\left[(r^\star(\x,\y_1)-r^\star(\x,\y_2)-\widehat{r}_\pi(\x,\y_1)+\widehat{r}_\pi(\x,\y_2))^2\right].
\end{equation}
Following~\citet{Song-2024-Importance}, we present policy performance in terms of this in-distribution pairwise error under the local coverage condition in the following definition.
\begin{definition}\label{def:local-rkl-coverage}
The reference policy $\pi_\textnormal{ref}$ satisfies local reverse-KL coverage at radius $\kappa>0$ with constant $C_\kappa>0$ if every policy $\mu$ satisfying $D_\textnormal{RKL}(\mu\|\pi_\textnormal{ref})\leq\kappa$ also satisfies
\begin{equation*}
\sup_{\x \in \operatorname{supp}(P_\textnormal{on})} \sup_{\y \in \VCal^\star}\tfrac{\mu(\y|\x)}{\pi_\textnormal{ref}(\y|\x)} \leq C_\kappa,
\end{equation*}
where we use the convention $\frac{0}{0}=0$.
\end{definition}
Local coverage in Definition~\ref{def:local-rkl-coverage} concerns policies within a reverse-KL neighborhood of $\pi_\textnormal{ref}$, which guarantees that restricting the learned policy to $\Pi_\tau$ can allow a performance guarantee to depend on coverage within that neighborhood. The reverse-KL constraint determines the class on which coverage is required but it does not guarantee the bounded density ratio.
\section{Main Results}\label{sec:main-results}
We study how to learn from noisy preference pairs that induce similar likelihoods for preferred and dispreferred responses. We first present the basic offline scheme, which replaces a first-order update driven by a predefined preference loss with a zeroth-order update driven by comparison oracles, and describe the practical offline scheme used for LLM fine-tuning. We then introduce online ComPO, which preserves the offline comparison direction and uses unlabeled current-policy generations for reverse-KL regularization. 

\subsection{Offline preference alignment}
The key idea behind ComPO is to use noisy preference pairs only to compare nearby policies. For a nonempty $S\subseteq D$, define
\begin{equation} \label{eq:pref-delta}
\begin{array}{rcl}
\Delta^+_S(\theta,\theta') & = & \tfrac{1}{|S|} \sum_{(\x,\y^+,\y^-) \in S}\left(\log \pi_{\theta'}(\y^+|\x)-\log \pi_\theta(\y^+|\x)\right), \\
\Delta^-_S(\theta,\theta') & = & \tfrac{1}{|S|} \sum_{(\x,\y^+,\y^-)\in S}\left(\log \pi_{\theta'}(\y^-|\x)-\log \pi_\theta(\y^-|\x)\right).
\end{array}
\end{equation}
We then provide the formulation of preference comparison oracle for LLM alignment below:
\begin{definition}[Preference comparison oracle]\label{def:pref-CO}
For a set $S \subseteq D$, the preference comparison oracle $\CCal^S_\pi(\theta,\theta'): \br^d \times \br^d \mapsto \{+1,-1\}$ is defined by
\begin{equation*}
\CCal^S_\pi(\theta,\theta') = \begin{cases}
-1, & \textnormal{if } \Delta^+_S(\theta,\theta')>0 \textnormal{ and } \Delta^-_S(\theta,\theta')<0,\\
+1, & \textnormal{otherwise}.
\end{cases}
\end{equation*}
\end{definition}
Thus, $\CCal^S_\pi(\theta,\theta')=-1$ means that $\theta'$ is preferred to $\theta$ according to the likelihood comparison induced by $S$. When $S$ contains one pair, this reduces to the pairwise oracle. When $S$ is a mini-batch, the oracle uses average preferred and dispreferred likelihood changes. Given a set of perturbations $\{\z_i\}_{i=1}^m$, ComPO queries $y_i=\CCal^S_\pi(\theta_t,\theta_t+r z_i)$ for $i=1,\ldots,m$ and applies the sparse 1-bit estimator from Eq.~\eqref{eq:1BGE} as follows, 
\begin{equation}\label{eq:compo-estimator}
\hat{\g} = \argmax_{\|\g\|_1 \leq \sqrt{s}, \|\g\| \leq 1} \sum_{i=1}^m y_i\z_i^\top\g.
\end{equation}
This is the only specialization of the comparison-oracle subroutine needed for offline ComPO. 

The following theorem establishes a best-iterate convergence guarantee for the basic offline scheme under smoothness, gradient sparsity, and oracle compatibility. 
\begin{theorem}\label{thm:offline}
Fix a nonempty comparison set $S \subseteq D$ and $1\leq s\leq d$. Suppose that there exists an $\ell$-smooth function $f:\br^d \to \br$, with $\ell>0$, that is bounded below and satisfies
\begin{enumerate}
\item For all $(\theta,\theta')$, we have $\CCal^S_\pi(\theta,\theta')=-1$ if $f(\theta')<f(\theta)$ and $\CCal^S_\pi(\theta,\theta')=1$ otherwise. 
\item The gradients of $f$ are approximately sparse: $\|\nabla f(\theta)\|_1\le \sqrt{s}\|\nabla f(\theta)\|$ for all $\theta \in \br^d$. 
\end{enumerate}
Let $\Delta>0$ satisfy $f(\theta_1)-\inf_{\theta \in \br^d} f(\theta) \leq \Delta$. For any $\epsilon,\Lambda \in (0,1)$, we choose
\begin{equation*}
T=\left\lceil\tfrac{10\ell\Delta}{\epsilon^2}\right\rceil,\quad \eta=\sqrt{\tfrac{2\Delta}{\ell T}},\quad r=\tfrac{\epsilon}{40\ell\sqrt d}, \quad m=\left\lceil c_m\left(s\log\left(\tfrac{2d}{s}\right)+\log\left(\tfrac{2T}{\Lambda}\right)\right)\right\rceil, 
\end{equation*}
where $c_m$ is a sufficiently large constant. Suppose that the perturbations at each iteration are drawn independently of the past and Eq.~\eqref{eq:compo-estimator} is solved exactly. Then, the iterates generated by Algorithm~\ref{alg:offline-basic} satisfy
\begin{equation*}
\PP\left(\min_{1 \leq t \leq T} \|\nabla f(\theta_t)\| < \epsilon\right) \geq 1-\Lambda .
\end{equation*}
Consequently, the total number of preference-comparison oracle calls is bounded by
\begin{equation*}
O\left(\left(1+\tfrac{\ell\Delta}{\epsilon^2}\right)\left(s\log\left(\tfrac{2d}{s}\right)
+\log\left(\tfrac{2+\ell\Delta\epsilon^{-2}}{\Lambda}\right)\right)\right).
\end{equation*} 
\end{theorem}
\begin{remark}
Theorem~\ref{thm:offline} provides a best-iterate convergence guarantee for the basic offline scheme under the stated assumptions. Since the objective $f$ is latent, its gradient norm is not available as a practical stopping criterion. The result nevertheless provides a theoretical benchmark: for fixed sparsity level $s$, the number of comparison queries depends only logarithmically on the ambient dimension. The practical implementation below approximates the basic estimator to accommodate the scale of LLM fine-tuning.
\end{remark}
\begin{algorithm}[!t]\small
\caption{Offline ComPO: Basic Scheme}
\label{alg:offline-basic}
\begin{algorithmic}[1]
\STATE \textbf{Input:} initial parameter $\theta_1 \in \br^d$, comparison set $S \subseteq D$, step size $\eta > 0$, sparsity ratio $s \ll d$, sampling radius $r > 0$, number of perturbations $m \geq 1$, and iteration number $T\geq 1$.
\FOR{$t=1,2,\ldots,T$}
\STATE Draw $m$ i.i.d. samples uniformly from the unit sphere in $\br^d$, denoted by $\{\z_i\}_{i=1}^m$.
\STATE Compute $y_i=\CCal^S_\pi(\theta_t,\theta_t+r\z_i)$ for $i=1,\ldots,m$.
\STATE Compute $\hat{\g}_t$ using Eq.~\eqref{eq:compo-estimator}.
\STATE Update $\theta_{t+1}=\theta_t-\eta\hat{\g}_t$.
\ENDFOR
\STATE \textbf{Output:} $\theta_{T+1}$.
\end{algorithmic}
\end{algorithm}
\begin{algorithm}[!t]\small
\caption{Offline ComPO: Practical Scheme}
\label{alg:offline-practical}
\begin{algorithmic}[1]
\STATE \textbf{Input:} initial parameter $\theta_1=[\bar{\theta};\theta^o_1]$, batches $\{S_t\}_{t=1}^T$, step size $\gamma$, sampling radius $r$, number of perturbations $m \geq 1$, clipping thresholds $\lambda_g,\lambda$, and iteration number $T\geq 1$.
\FOR{$t=1,2,\ldots,T$}
\STATE Draw $m$ i.i.d. samples $\{\z_i\}_{i=1}^m$ uniformly from the unit sphere in $\br^{d_o}$. 
\STATE Query $y_i=\CCal_\pi^{S_t}([\bar{\theta}; \theta^o_t], [\bar{\theta}; \theta^o_t+r z_i])$ for all $i=1,\ldots,m$. 
\STATE Set $\textbf{u}_t=\sum_{i=1}^m y_i\z_i$. If $\textbf{u}_t \neq 0$, set $\hat{\g}^o_t=\textbf{u}_t/\|\textbf{u}_t\|$. Otherwise, set $\hat{\g}^o_t=0$.
\STATE Clip $\hat{\g}^o_t$ by zeroing out entries whose magnitude is less than $\lambda_g$.
\STATE Set $p_t=\frac{|\{i:y_i=-1\}|}{m}$.
\IF{$p_t>\lambda$}
\STATE $\theta^o_{t+1}=\theta^o_t-\gamma p_t\hat{\g}^o_t$.
\ELSE
\STATE $\theta^o_{t+1}=\theta^o_t$.
\ENDIF
\ENDFOR
\STATE \textbf{Output:} $\theta_{T+1}=[\bar{\theta};\theta^o_{T+1}]$.
\end{algorithmic}
\end{algorithm}

\paragraph{Practical scheme.} Applying the basic scheme to all model parameters is computationally expensive for LLMs. We therefore perturb only the output-layer weights
$\theta^o \in \br^{d_o}$ and freeze the remaining parameters $\bar{\theta}$, so that $\theta=[\bar{\theta};\theta^o]$. We also replace the exact solution of Eq.~\eqref{eq:compo-estimator} with a normalized sum of signed perturbations followed by entry-wise clipping.

The practical pipeline partitions the dataset using the reference model. In particular, we define
\begin{equation}\label{eq:noisy-margin}
D_\textnormal{noisy} = \left\{(\x,\y^+,\y^-)\in D: |\log\pi_\textnormal{ref}(\y^+|\x)
-\log\pi_\textnormal{ref}(\y^-|\x)| \leq \delta_\textnormal{margin}\right\},
\end{equation}
and let $D_\textnormal{clean}=D\setminus D_\textnormal{noisy}$. The term noisy refers to this low-margin subset and does not presume that its preference labels are incorrect.
We first apply a direct alignment method, such as DPO or SimPO, to $D_\textnormal{clean}$ and then apply Algorithm~\ref{alg:offline-practical} to $D_\textnormal{noisy}$. For DPO in the first stage, we denote the resulting procedure by DPO$_\textnormal{clean}$+ComPO.

\subsection{Online ComPO}\label{sec:online-kl-compo}
We introduce an online extension of ComPO that retains the offline comparison mechanism and uses unlabeled policy generations for reverse-KL control. Following~\citet{Song-2024-Importance}, we restrict the policy to the class $\Pi_\tau$ in Eq.~\eqref{eq:kl-ball}, so that the analysis requires coverage only within this neighborhood. Since the update direction is obtained from comparisons rather than the gradient of an explicit preference loss, the basic scheme implements this restriction through a feasibility check on each candidate update. The practical scheme uses the samples from the current policy to adjust the step size.

At iteration $t$, we compute the same comparison direction $\hat{\g}_t$ as in Algorithm~\ref{alg:offline-basic} and form a single candidate using a fixed step size $\eta>0$:
\begin{equation}\label{eq:basic-online-rkl}
\tilde{\theta}_{t+1}=\theta_t-\eta\hat{\g}_t, \quad \tilde{D}_t=D_\textnormal{RKL}(\pi_{\tilde{\theta}_{t+1}}\|\pi_\textnormal{ref}).
\end{equation}
Given a reverse-KL radius $\tau>0$, we accept the candidate if it is feasible and otherwise leave the policy unchanged:
\begin{equation}\label{eq:basic-online-step}
\theta_{t+1}= \begin{cases}
\tilde{\theta}_{t+1}, & \textnormal{if } \tilde{D}_t \leq \tau,\\ \theta_t, & \textnormal{otherwise}.
\end{cases}
\end{equation}
The basic scheme evaluates the candidate policy's reverse KL exactly. Starting from a feasible policy, the accept-or-reject rule preserves feasibility by retaining the previous iterate whenever the candidate falls outside $\Pi_\tau$.

\begin{algorithm}[!t]\small
\caption{Online ComPO: Basic Scheme}\label{alg:basic-online}
\begin{algorithmic}[1]
\STATE \textbf{Input:} initial parameter $\theta_1 \in \br^d$ satisfying $\pi_{\theta_1} \in \Pi_\tau$, comparison set $S \subseteq D$, online prompt distribution $P_\textnormal{on}$, reference policy $\pi_\textnormal{ref}$, step size $\eta>0$, reverse-KL radius $\tau>0$, sparsity ratio $s\ll d$, sampling radius $r>0$, number of perturbations $m \geq 1$, and iteration number $T \geq 1$.
\FOR{$t=1,2,\ldots,T$}
\STATE Draw $m$ i.i.d. samples uniformly from the unit sphere in $\br^d$, denoted by $\{\z_i\}_{i=1}^m$.
\STATE Compute $y_i=\CCal^S_\pi(\theta_t,\theta_t+r\z_i)$ for $i=1,\ldots,m$.
\STATE Compute $\hat{\g}_t$ using Eq.~\eqref{eq:compo-estimator}.
\STATE Form $\tilde{\theta}_{t+1}$ and evaluate $\tilde{D}_t$ by Eq.~\eqref{eq:basic-online-rkl}.
\STATE Set $\theta_{t+1}$ according to Eq.~\eqref{eq:basic-online-step}.
\ENDFOR
\STATE \textbf{Output:} $\theta_{T+1}$.
\end{algorithmic}
\end{algorithm}

The following theorem establishes feasibility and relates in-distribution pairwise reward accuracy to policy performance under local coverage.
\begin{theorem}\label{thm:online}
Fix $\beta,\tau>0$. Suppose that Algorithm~\ref{alg:basic-online} evaluates each candidate policy's reverse KL exactly. Then, the generated iterates satisfy $\pi_{\theta_t} \in \Pi_\tau$ for all $t=1,\ldots,T+1$. If $\pi_\textnormal{ref}$ satisfies local reverse-KL coverage at radius $\tau$ with constant $C_\tau$, we have
\begin{equation*}
\sup_{\pi\in\Pi_\tau}J_\beta(\pi) - J_\beta(\pi_{\theta_t}) \leq C_\tau\sqrt{\operatorname{err}(\pi_{\theta_t})}, \quad \textnormal{for all } t=1,\ldots,T+1. 
\end{equation*}
For any $\epsilon>0$, an iterate satisfying $\operatorname{err}(\pi_{\theta_t}) \leq \epsilon$ satisfies $\sup_{\pi\in\Pi_\tau}J_\beta(\pi)-J_\beta(\pi_{\theta_t})
\leq C_\tau\sqrt{\epsilon}$. 
\end{theorem}
Theorem~\ref{thm:online} combines the feasibility preservation with a coverage-based performance bound following~\citet{Song-2024-Importance}. The reverse-KL constraint restricts the policies under consideration to $\Pi_\tau$, so that this guarantee requires coverage within the neighborhood rather than over the entire policy class. Within this neighborhood, smaller pairwise reward error gives a tighter performance bound. 
\begin{algorithm}[!t]\small
\caption{Online ComPO: Practical Scheme}\label{alg:online-practical}
\begin{algorithmic}[1]
\STATE \textbf{Input:} initial parameter $\theta_1=[\bar{\theta};\theta^o_1]$, preference dataset $D$, online prompts $\XCal_\textnormal{on}$, reference policy $\pi_\textnormal{ref}$, margin threshold $\delta_\textnormal{margin}$, step-size scale $\gamma>0$, damping strength $\rho\geq0$, threshold $\tau_p \geq 0$, sampling radius $r>0$, number of perturbations $m\geq1$, online batch size $B\geq1$, clipping thresholds $\lambda_g,\lambda>0$,  iteration number $T \geq 1$, re-sampling window $n\geq1$, and replay ratio $\alpha\in[0,1]$.
\STATE Construct $D_\textnormal{noisy}$ using Eq.~\eqref{eq:noisy-margin}.
\STATE Initialize the replay buffer $\RCal \leftarrow \emptyset$ and the current successful-batch buffer $\ACal \leftarrow \emptyset$.
\FOR{$t=1,2,\ldots,T$}
{\IF{$t>1$ and $(t-1)\bmod n=0$}
\STATE Set $\RCal \leftarrow \ACal$ and $\ACal \leftarrow \emptyset$.
\ENDIF
\STATE Draw a replay indicator $b_t \sim \operatorname{Bernoulli}(\alpha)$.
\IF{$b_t=1$ and $\RCal \neq \emptyset$} 
\STATE Sample a previously successful preference mini-batch $S_t$ uniformly from $\RCal$.
\ELSE 
\STATE Sample a new noisy preference mini-batch $S_t\subseteq D_{\textnormal{noisy}}$.
\ENDIF}
\STATE Draw $m$ i.i.d. samples uniformly from the unit sphere in $\br^{d_o}$, denoted by $\{\z_i\}_{i=1}^m$.
\STATE Query $y_i=\CCal_\pi^{S_t}([\bar{\theta};\theta^o_t],[\bar{\theta};\theta^o_t+r\z_i])$ for $i=1,\ldots,m$.
\STATE Set $\textbf{u}_t=\sum_{i=1}^m y_i\z_i$ and $\hat{\g}^o_t=\textbf{u}_t/\|\textbf{u}_t\|$ if $\textbf{u}_t\neq0$; otherwise set $\hat{\g}^o_t=0$. Clip $\hat{\g}^o_t$ by zeroing out entries whose magnitude is less than $\lambda_g$.
\STATE Sample $\{\tilde{\x}_j\}_{j=1}^B \subseteq \XCal_\textnormal{on}$, generate $\tilde{\y}_j \sim \pi_{\theta_t}(\cdot|\tilde{\x}_j)$, and compute $\hat{d}_t$ and $\gamma_t$ using Eq.~\eqref{eq:length-normalized-rkl}-\eqref{eq:online-damped-stepsize}.
\STATE Set $p_t=\frac{|\{i:y_i=-1\}|}{m}$.
\IF{$p_t>\lambda$}
\STATE $\theta^o_{t+1}=\theta^o_t-\gamma_t p_t\hat{\mathbf{g}}^o_t$.
\STATE Add the accepted preference mini-batch to the current buffer: $\ACal \leftarrow \ACal \cup \{S_t\}$.
\ELSE
\STATE $\theta^o_{t+1}=\theta^o_t$.
\ENDIF
\ENDFOR
\STATE \textbf{Output:} $\theta_{T+1}=[\bar{\theta};\theta^o_{T+1}]$.
\end{algorithmic}
\end{algorithm}

\paragraph{Practical scheme.} While the basic scheme evaluates reverse KL at the candidate policy, the practical scheme samples from the current policy and uses a length-normalized statistic to damp the update. For independent prompts $\tilde{\x}_j \sim P_\textnormal{on}$ and responses $\tilde{\y}_j \sim \pi_{\theta_t}(\cdot \mid \tilde{\x}_j)$,
the sequence-level estimator
\begin{equation}\label{eq:seq-rkl-estimator}
\widehat{D}_t^\textnormal{seq} = \tfrac{1}{B}\sum_{j=1}^B
\log\left(\tfrac{\pi_{\theta_t}(\tilde{\y}_j \mid \tilde{\x}_j)}{\pi_\textnormal{ref}(\tilde{\y}_j \mid \tilde{\x}_j)}\right)
\end{equation}
is unbiased for $D_\textnormal{RKL}(\pi_{\theta_t}\|\pi_\textnormal{ref})$. In practice, we use
\begin{equation}\label{eq:length-normalized-rkl}
\hat{d}_t = \tfrac{1}{B}\sum_{j=1}^B \tfrac{\log\pi_{\theta_t}(\tilde{\y}_j \mid \tilde{\x}_j)-\log\pi_\textnormal{ref}(\tilde{\y}_j \mid \tilde{\x}_j)}{\max\{1,|\tilde{\y}_j|\}}.
\end{equation}
Length normalization changes the population quantity being estimated. In particular, $\hat d_t$ is a signed statistic and its population counterpart needs not be nonnegative. We set
\begin{equation}\label{eq:online-damped-stepsize}
\gamma_t = \tfrac{\gamma}{1+\rho\max\{\hat{d}_t-\tau_p,0\}},
\end{equation}
where $\tau_p$ is the threshold for the length-normalized statistic. As such, the online samples only change the step size, not the comparison oracle or the preference labels.

We divide training into consecutive blocks of $n$ iterations. At the start of each block after the first, the replay buffer is replaced by the mini-batches that passed the update gate $p_t>\lambda$ in the preceding completed block. At each iteration, with probability
$\alpha$, we sample uniformly from this buffer when it is nonempty; otherwise, we sample a new mini-batch from $D_\textnormal{noisy}$. Revisited mini-batches use fresh perturbations around the current parameters rather than reusing previous update directions. Here, ``successful" means only that the comparison gate was passed. Section~\ref{sec:online} evaluates the empirical effect of combining replay with online damping.

Algorithm~\ref{alg:online-practical} is motivated by the principle used in Algorithm~\ref{alg:basic-online}, but cannot be covered by Theorem~\ref{thm:online}. In particular, the length-normalized quantity in Eq.~\eqref{eq:length-normalized-rkl} is not the sequence-level reverse KL in Eq.~\eqref{def:rkl}, and Eq.~\eqref{eq:online-damped-stepsize} does not enforce the hard constraint $\pi\in\Pi_\tau$. These are practical heuristics whose effect is evaluated empirically in Section~\ref{sec:online}.

\section{Experiments}\label{sec:exp}
We investigate the effectiveness of ComPO on aligning the LLMs. First, we evaluate offline scheme as an augmentation to DPO and its variants, where it extracts the directions from noisy preference pairs. Second, we study the offline design choices and the scaling behavior with respect to perturbations, perturbed layers and noisy pairs. Third, we evaluate online scheme, which uses unlabeled current-policy generations to damp the step through a reverse-KL proxy. Unless otherwise stated, the main tables report point estimates from the reported runs and the ablation tables explicitly report variation across repeated runs.
\setlength{\tabcolsep}{2pt}
\begin{table}[!t]
\caption{\footnotesize{Evaluation on AlpacaEval~2, Arena-Hard, and MT-Bench across four model configurations. LC and WR denote length-controlled win rate and raw win rate, respectively. Turn-1 and Turn-2 are the MT-Bench scores for the initial and follow-up questions. ``PA'' denotes the pre-alignment supervised or instruction-fine-tuned checkpoint before DPO training.}}
\centering
\resizebox{\textwidth}{!}{
\begin{tabular}{lcccccccccccc}
\toprule
\multirow{3}{*}{\textbf{Method}} & \multicolumn{6}{c}{\textbf{Mistral-7B-Base}} & \multicolumn{6}{c}{\textbf{Mistral-7B-Instruct}} \\ 
\cmidrule(lr){2-7} 
\cmidrule(lr){8-13} 
& \multicolumn{2}{c}{\textbf{AlpacaEval 2}} & \multicolumn{1}{c}{\textbf{Arena-Hard}} & \multicolumn{3}{c}{\textbf{MT-Bench}} & \multicolumn{2}{c}{\textbf{AlpacaEval 2}} & \multicolumn{1}{c}{\textbf{Arena-Hard}} & \multicolumn{3}{c}{\textbf{MT-Bench}} \\ 
\cmidrule(lr){2-3} 
\cmidrule(lr){4-4}
\cmidrule(lr){5-7} 
\cmidrule(lr){8-9}
\cmidrule(lr){10-10}
\cmidrule(lr){11-13} 
& {\footnotesize \bf LC (\%)} & {\footnotesize \bf WR (\%)} & {\footnotesize \bf WR (\%)} & {\footnotesize \bf Turn-1} & {\footnotesize \bf Turn-2} & {\footnotesize \bf Avg.} & {\footnotesize \bf LC (\%)} & {\footnotesize\bf WR (\%)} & {\footnotesize \bf WR (\%)} & {\footnotesize \bf Turn-1} & {\footnotesize \bf Turn-2} & {\footnotesize \bf Avg.} \\ \midrule
PA & 7.33 & 4.48 & 1.1 & 6.10 & 5.04 & 5.57  & 16.54 & 12.43 & 10.9 & 6.19 & 5.10 & 5.65 \\
DPO &  9.71 & 6.27 & 2.9 & 6.20 & \textbf{5.38} & \textbf{5.79} & 24.14 & 16.71 & \textbf{14.4} & 6.28 & 5.42 & 5.86 \\
DPO$_{\textnormal{clean}}$ & 9.41 & 6.52 & 3.0 & 6.18 & 5.22 & 5.70  & 23.89 & 16.15 & 14.2 & 6.11 & 5.34 & 5.73 \\ \midrule
DPO$_{\textnormal{clean}}$+ComPO & \textbf{11.66} & \textbf{6.55} & \textbf{3.2} & \textbf{6.22} & 5.32 & 5.77 & \textbf{26.17} & \textbf{18.32} & 10.5 & \textbf{7.78} & \textbf{7.63} & \textbf{7.69} \\

\midrule[.7pt]
\multirow{3}{*}{\textbf{Method}} & \multicolumn{6}{c}{\textbf{Llama-3-8B-Base}} & \multicolumn{6}{c}{\textbf{Llama-3-8B-Instruct}} \\ 
\cmidrule(lr){2-7}
\cmidrule(lr){8-13}
& \multicolumn{2}{c}{\textbf{AlpacaEval 2}} & \multicolumn{1}{c}{\textbf{Arena-Hard}} & \multicolumn{3}{c}{\textbf{MT-Bench}} & \multicolumn{2}{c}{\textbf{AlpacaEval 2}} & \multicolumn{1}{c}{\textbf{Arena-Hard}} & \multicolumn{3}{c}{\textbf{MT-Bench}} \\ 
\cmidrule(lr){2-3}
\cmidrule(lr){4-4}
\cmidrule(lr){5-7}
\cmidrule(lr){8-9}
\cmidrule(lr){10-10}
\cmidrule(lr){11-13}
& {\footnotesize \bf LC (\%)} & {\footnotesize \bf WR (\%)} & {\footnotesize \bf WR (\%)} & {\footnotesize \bf Turn-1} & {\footnotesize \bf Turn-2} & {\footnotesize \bf Avg.} & {\footnotesize \bf LC (\%)} & {\footnotesize\bf WR (\%)} & {\footnotesize \bf WR (\%)} & {\footnotesize \bf Turn-1} & {\footnotesize \bf Turn-2} & {\footnotesize \bf Avg.} \\ \midrule
PA & 3.21 & 7.97 & 4.1 & 6.53 & 5.66 & 6.10  & 24.06 & 23.69 & 20.8 & 8.22 & 7.57 & 7.90 \\
DPO & 4.14 & 10.43 & \textbf{12.1} & 6.61 & 5.85 & 6.23 & 32.59 & 31.99 & 22.9 & 8.30 & 7.55 & 7.93 \\
DPO$_{\textnormal{clean}}$ & 4.28 & 9.81 & 12.0 & \textbf{6.64} & 6.01 & 6.33 & 32.92 & 32.42 & 22.9 & 8.26 & 7.63 & 7.94 \\ \midrule
DPO$_{\textnormal{clean}}$+ComPO & \textbf{5.39} & \textbf{10.93} & \textbf{12.1} & 6.60 & \textbf{6.28} & \textbf{6.44} & \textbf{35.79} & \textbf{35.03} & \textbf{23.1} & \textbf{8.39} & \textbf{7.71} & \textbf{8.05} \\ \bottomrule
\end{tabular}
} 
\label{tab:main}
\vspace{-1em}
\end{table}

\subsection{Offline training for augmenting DPO and SimPO}\label{sec:offline}
We identify clean and noisy preference pairs using the margin threshold $\delta_\textnormal{margin}=3$. For Mistral-7B models, we set $r=0.0005$, $m=1600$, $\lambda_g=0.00022$, and $\lambda=0.2$. For Llama-3-8B models and Gemma-2-9B-it, we set $r=0.00075$, $m=1800$, $\lambda_g=0.00008$, and $\lambda=0.2$. We use UltraFeedback\footnote{\url{https://huggingface.co/datasets/HuggingFaceH4/ultrafeedback_binarized}}~\citep{Cui-2024-Ultrafeedback} throughout the offline experiments. We initialize from the supervised fine-tuned \textbf{Base} and \textbf{Instruct} models used in~\citet{Meng-2024-SimPO}: Mistral-7B Base and Instruct\footnote{\url{https://huggingface.co/alignment-handbook/zephyr-7b-sft-full}}, Llama-3-8B Base\footnote{\url{https://huggingface.co/princeton-nlp/Llama-3-Base-8B-SFT}} and Instruct\footnote{\url{https://huggingface.co/meta-llama/Meta-Llama-3-8B-Instruct}}, and Gemma-2-9B-it\footnote{\url{https://huggingface.co/google/gemma-2-9b-it}}. All ComPO runs use 30 NVIDIA A40 GPUs, each with 46 GB of memory.

We follow the evaluation protocol of~\citet{Meng-2024-SimPO} and evaluate on AlpacaEval~2-v0.6.6~\citep{Li-2023-AlpacaEval}, Arena-Hard~\citep{Li-2024-Crowdsourced}, and MT-Bench~\citep{Zheng-2023-Judging}. For AlpacaEval~2, GPT-4 Turbo serves as both baseline and judge models. The judge compares each model response with the baseline response, and we report raw win rate (WR) and length-controlled win rate (LC)~\citep{Dubois-2024-Length}. LC adjusts judged preferences for response length and a higher LC score does not by itself establish shorter responses. For Arena-Hard, the baseline is GPT-4-0314 and the judge is GPT-4 Turbo. We report WR. For MT-Bench, GPT-4 scores multi-turn Q\&A responses on a 10-point scale. We report the scores for the initial question (Turn-1), the follow-up question (Turn-2), and their average.
\setlength{\tabcolsep}{2pt}
\begin{table*}[!t] \small
\centering 
\caption{\footnotesize{Pairwise log-likelihoods in three independent trials for $\gamma \in \{0.1, 1\}$, with all other hyperparameters fixed at their default values. Each cell reports $(\log\pi_\theta(\y^+ | \x), \log\pi_\theta(\y^- | \x))$ after one training run; the initial values appear in the model headers. The trials use independently sampled perturbations $\{\z_i\}_{1 \leq i \leq m}$. Across the reported trials, the preferred-response log-likelihood is nondecreasing and the dispreferred-response log-likelihood is nonincreasing.}}
\label{table:likelihood_displacement} \vspace*{-.5em}
\centering
\resizebox{0.67\linewidth}{!}{%
\begin{tabular}{@{}cccc@{}}
\toprule
\multicolumn{4}{c}{\textbf{Llama-3-Instruct-8B} \ $(\log \pi_\theta(\y^+ | \x), \log\pi_\theta(\y^- | \x))=(-46.761, -47.410)$} \\ \midrule
$\gamma$ & Trial 1 & Trial 2 & Trial 3 \\ \midrule
0.1 & $(-46.744,\,-47.411)$ & $(-46.760,\,-47.411)$ & $(-46.759,\,-47.410)$ \\
1 & $(-46.728,\,-47.520)$ & $(-46.743,\,-47.525)$ & $(-46.753,\,-47.517)$ \\
\addlinespace[0.1em] \midrule
\multicolumn{4}{c}{\textbf{Gemma-2-9B-it} \ $(\log\pi_\theta(\y^+ | \x), \log\pi_\theta(\y^- | \x))=(-133.122, -134.557)$} \\ \midrule
$\gamma$ & Trial 1 & Trial 2 & Trial 3 \\ \midrule
0.1 & $(-133.122,\,-134.557)$ & $(-133.122,\,-134.557)$ & $(-133.121,\,-134.557)$ \\
1 & $(-133.059,\,-134.562)$ & $(-133.122,\,-134.564)$ & $(-133.112,\,-134.565)$ \\ \bottomrule
\end{tabular}
}
\vspace{-1em}
\end{table*}

\paragraph{DPO with ComPO.} We split the data into clean and noisy subsets using the margin criterion in Eq.~\eqref{eq:noisy-margin}. Starting from the SFT model, we train on all pairs to obtain DPO and on only the clean pairs to obtain DPO$_{\textnormal{clean}}$. Following~\citet{Meng-2024-SimPO}, both models are trained for one epoch. We initialize ComPO from DPO$_{\textnormal{clean}}$ and run it for one epoch with 100 iterations over noisy pairs, yielding DPO$_{\textnormal{clean}}$+ComPO.

We summarize the results in Table~\ref{tab:main} and report three key observations. First, filtering low-margin pairs alone does not uniformly improve DPO: DPO$_{\textnormal{clean}}$ is comparable to DPO overall and performs better only for some initializations, such as Llama-3-Instruct-8B. The log-likelihood margin therefore appears to be an imperfect proxy for pair ambiguity; richer criteria such as the CHES score~\citep{Razin-2025-Unintentional} may separate pairs more accurately. Nevertheless, the margin is inexpensive to compute, and ComPO extracts useful information from the pairs that it filters out. Second, gains are especially consistent in AlpacaEval~2 LC, indicating improved judged performance after adjustment for response length. We interpret these scores separately from the response-length measurements reported below. Third, ComPO uses only the first 100 noisy pairs, yet improves most model-benchmark combinations. As such, a small set of low-margin pairs can contain useful alignment information when processed through comparison oracles.

The main exception is Arena-Hard for Mistral-7B-Instruct, where DPO scores $14.4$ and DPO$_{\textnormal{clean}}$+ComPO scores $10.5$; for the two Llama configurations, the scores are tied or nearly tied. An explanation is that Arena-Hard reports raw rather than length-controlled win rate and can therefore favor longer generations~\citep{Meng-2024-SimPO}. For Mistral-7B-Instruct, the average response length is $513$ for DPO and $468$ for DPO$_{\textnormal{clean}}$+ComPO. This difference is consistent with the lower Arena-Hard score and the stronger AlpacaEval~2 LC score, although it does not by itself establish causality.

We also inspect whether the comparison oracle moves the likelihoods of each noisy pair in the intended direction. In Table~\ref{table:likelihood_displacement}, we summarize three independent trials for $\gamma \in \{0.1,1\}$ on Llama-3-Instruct-8B and Gemma-2-9B-it. For example, with Llama-3-Instruct-8B and $\gamma=1$, the first trial changes the pair from $(-46.761,-47.410)$ to $(-46.728,-47.520)$: the preferred response becomes more likely, while the dispreferred response becomes less likely. Thus, for the two reported models, the oracle-based update moves the pairwise likelihoods in the desired direction or leaves them unchanged. This diagnostic is an in-training sanity check rather than a population-level performance guarantee.

The thresholds $\lambda_g$ and $\lambda$ limit the coordinates and iterations on which the practical scheme updates the model. Very large step sizes can still destabilize the practical scheme, while Theorem~\ref{thm:offline} analyzes the step size only for the basic
scheme. Section~\ref{sec:online} considers adaptive step-size control based on current-policy generations.

\paragraph{SimPO with ComPO.} ComPO is not tied to DPO. We apply it directly to existing, well-tuned SimPO checkpoints~\citep{Meng-2024-SimPO} and use the training and evaluation configuration described at the beginning of Section~\ref{sec:offline}. Table~\ref{tab:main_simpo} shows that SimPO+ComPO improves both AlpacaEval~2 metrics for all three models. On Arena-Hard, it improves Mistral-7B-Instruct and Llama-3-8B-Instruct and matches Gemma-2-9B-it. The MT-Bench average also increases slightly for each model. These results show that ComPO augments other direct alignment methods without changing its original training objective.

\begin{table}[!t] \small
\caption{\footnotesize{Applying ComPO to existing SimPO checkpoints across models and benchmarks.}}
\centering
\resizebox{0.85\textwidth}{!}{
\begin{tabular}{c|ccccccc} \toprule
\multirow{2}{*}{\textbf{Model}} & \multirow{2}{*}{\textbf{Method}} & \multicolumn{2}{c}{\textbf{AlpacaEval 2}} & \multicolumn{1}{c}{\textbf{Arena-Hard}} & \multicolumn{3}{c}{\textbf{MT-Bench}} \\
\cmidrule(lr){3-4}\cmidrule(lr){5-5}\cmidrule(lr){6-8}
& & {\scriptsize \bf LC (\%)} & {\scriptsize \bf WR (\%)} & {\scriptsize \bf WR (\%)} & {\scriptsize \bf Turn-1} & {\scriptsize \bf Turn-2} & {\scriptsize \bf Avg.} \\ \midrule
\multirow{2}{*}{\textbf{Mistral-7B-Instruct}} & SimPO & 40.22 & 41.18 & 20.8 & \textbf{7.94} & 7.31 & 7.62 \\
& SimPO + ComPO & \textbf{42.27} & \textbf{43.17} & \textbf{22.0} & 7.83 & \textbf{7.46} & \textbf{7.64} \\ \midrule
\multirow{2}{*}{\textbf{Llama-3-8B-Instruct}} & SimPO & 48.71 & 43.66 & 36.3 & 7.91 & 7.42 & 7.66 \\
& SimPO + ComPO & \textbf{49.53} & \textbf{45.03} & \textbf{37.3} & \textbf{7.94} & \textbf{7.45} & \textbf{7.70} \\ \midrule
\multirow{2}{*}{\textbf{Gemma-2-9B-it}} & SimPO & 60.36 & 55.59 & \textbf{61.1} & \textbf{9.07} & 8.47 & 8.77 \\
& SimPO + ComPO & \textbf{62.42} & \textbf{57.20} & \textbf{61.1} & 8.99 & \textbf{8.58} & \textbf{8.79} \\ \bottomrule
\end{tabular}
}
\label{tab:main_simpo}
\vspace{-1em}
\end{table}

\subsection{Ablation studies}\label{sec:abalation}
\paragraph{Number of perturbations.} The number of perturbations controls how many directions the comparison oracle evaluates. We vary $m$ while holding the remaining hyperparameters fixed and use Mistral-7B-Instruct for this study. As $m$ increases from $800$ to $5400$, the mean WR and LC improve, with diminishing gains at larger $m$ (Table~\ref{tab:perturb_m}). This is consistent with a more accurate gradient estimate from additional perturbations, although the computation time increases. Peak memory remains unchanged because ComPO accumulates a running average rather than storing all perturbation vectors (see Line~5 of Algorithm~\ref{alg:offline-practical}).

We also investigate whether ComPO scales beyond output-layer perturbations. Keeping all other settings fixed, we perturb the MLPs in layers 30--31 together with the output layer of Mistral-7B-Instruct. Table~\ref{tab:multilayer_perturb} uses GPT-4.1 as the Arena-Hard judge, and perturbing three layers improves all three reported metrics. The larger search space has a modest systems cost in this setup: peak GPU memory increases from $16.3$ GB to $16.7$ GB, and $600$ perturbations take $60$ seconds rather than $50$ seconds.

\paragraph{Gradient threshold and number of noisy pairs.} ComPO uses the entry threshold $\lambda_g$ to update only gradient entries with sufficiently large magnitude. We vary $\lambda_g$ with $m=3300$ on Mistral-7B-Instruct (Table~\ref{tab:lambda_g}). The strongest results occur when approximately $1\%$--$6\%$ of the entries are retained. Retaining many small entries or filtering almost all entries leads to lower performance. We then increase the number of noisy pairs from $100$ to $300$. Table~\ref{tab:noisy_pairs_scale} shows higher mean performance on both AlpacaEval~2 metrics and Arena-Hard, indicating that ComPO continues to benefit from additional low-margin pairs.
\begin{table}[!t] 
\centering
\caption{\footnotesize{Effect of the number of perturbations $m$ on AlpacaEval~2. Entries are mean $\pm$ standard deviation over five runs, with the best run in parentheses.}} \label{tab:perturb_m}
\resizebox{0.95\columnwidth}{!}{%
\begin{tabular}{l|cccc} \toprule
\textbf{Perturbation ($m$)} & \textbf{800} & \textbf{1600} & \textbf{3300} & \textbf{5400} \\ \midrule
\textbf{AlpacaEval 2-WR} \% & $17.32 \pm 0.86\ (17.94)$ & $17.50 \pm 0.65\ (18.32)$ & $19.21 \pm 0.58\ (\textbf{20.25})$ & $19.69 \pm 0.36\ (20.07)$ \\
\textbf{AlpacaEval~2-LC} \% & $24.72 \pm 1.02\ (25.12)$ & $25.02 \pm 0.91\ (26.17)$ & $25.91 \pm 0.95\ (27.14)$ & $26.49 \pm 0.81\ (\textbf{27.20})$ \\ \bottomrule
\end{tabular}%
}
\vspace{-.5em}
\end{table}
\begin{table}[!t]
\vspace{-.5em}
\centering
\caption{\footnotesize{Effect of perturbing multiple layers. We report AlpacaEval~2 WR and LC and Arena-Hard WR. Entries are mean $\pm$ standard deviation over five runs, with the best run in parentheses.}} \label{tab:multilayer_perturb}
\resizebox{0.95\columnwidth}{!}{%
\begin{tabular}{l|ccc}
\toprule
\textbf{Layers perturbed (\# params)} & \textbf{AlpacaEval~2-WR \%} & \textbf{AlpacaEval~2-LC \%} & \textbf{Arena-Hard (GPT-4.1)-WR \%} \\ \midrule
1 (0.13B) & $17.50 \pm 0.65\ (18.32)$ & $25.02 \pm 0.91\ (26.17)$ & $10.80 \pm 0.21\ (11.0)$ \\
3 (0.25B) & $18.19 \pm 0.81\ (19.38)$ & $26.00 \pm 0.89\ (27.09)$ & $11.26 \pm 0.36\ (11.7)$ \\ \bottomrule
\end{tabular}%
}
\vspace{-1em}
\end{table}
\begin{table}[!t]
\centering
\caption{\footnotesize{Effect of the gradient-entry threshold $\lambda_g$ on AlpacaEval~2. Entries are mean $\pm$ standard deviation over five runs, with the best run in parentheses.}} \label{tab:lambda_g}
\resizebox{\columnwidth}{!}{%
\begin{tabular}{l|ccccc} \toprule
\textbf{\textbf{$\lambda_g$}} & \textbf{$0$} & \textbf{$4{\times}10^{-5}$} & \textbf{$1.8{\times}10^{-4}$} & \textbf{$2.2{\times}10^{-4}$} & \textbf{$2.5{\times}10^{-4}$} \\ \midrule
\textbf{Percentage of gradient entries updated} & $100\%$ & $63\%$ & $6\%$ & $1\%$ & $0.15\%$ \\
\textbf{AlpacaEval 2-WR \%} & $15.72 \pm 0.77\ (16.34)$ & $16.02 \pm 0.69\ (16.69)$ & $19.02 \pm 0.62\ (20.15)$ & $19.21 \pm 0.58\ (20.25)$ & $16.10 \pm 0.11\ (16.21)$ \\
\textbf{AlpacaEval 2-LC \% }& $23.42 \pm 1.03\ (24.28)$ & $24.01 \pm 0.91\ (25.10)$ & $26.06 \pm 0.81\ (27.27)$ & $25.91 \pm 0.95\ (27.14)$ & $23.82 \pm 0.23\ (24.00)$ \\ \bottomrule
\end{tabular}%
}
\vspace{-1em}
\end{table}
\begin{table}[!t]
\centering
\caption{\footnotesize{Effect of increasing the number of noisy preference pairs used by ComPO. Entries are mean $\pm$ standard deviation, with the best run in parentheses.}} \label{tab:noisy_pairs_scale}
\resizebox{0.95\columnwidth}{!}{%
\centering
\begin{tabular}{l|ccc} \toprule
\textbf{Number of noisy pairs} & \textbf{AlpacaEval 2-WR \%} & \textbf{AlpacaEval 2-LC \%} & \textbf{Arena-Hard (GPT-4.1)-WR \%} \\ \midrule
100 & $19.21 \pm 0.58\ (20.25)$ & $25.91 \pm 0.95\ (27.14)$ & $11.02 \pm 0.13\ (11.2)$ \\
300 & $20.07 \pm 0.99\ (21.35)$ & $26.28 \pm 0.81\ (27.59)$ & $11.76 \pm 0.30\ (12.1)$ \\ \bottomrule
\end{tabular}%
} \vspace*{-1em}
\end{table}

\begin{table}[!t]
\centering
\caption{\footnotesize{Applying ComPO directly to DPO checkpoints without training DPO only on the clean subset. AE, AH, and MT denote AlpacaEval~2, Arena-Hard, and MT-Bench, respectively.}} \label{tab:mistral_main}
\resizebox{\linewidth}{!}{%
\begin{tabular}{lcccccc} \toprule
\textbf{Method} & \textbf{AE LC (\%)} & \textbf{AE WR (\%)} & \textbf{AH (GPT-4.1) WR (\%)} & \textbf{MT Turn 1} & \textbf{MT Turn 2} & \textbf{MT Avg} \\ \midrule
DPO & 24.14 & 16.71 & 10.40 & 6.28 & 5.42 & 5.86 \\
DPO + ComPO & 27.03 & 20.85 & 11.40 & 7.80 & 7.61 & 7.71 \\ \midrule
DPO (clean) & 23.89 & 16.15 & 10.50 & 6.11 & 5.34 & 5.73 \\
DPO (clean) + ComPO    & 27.14 & 20.25 & 11.20 & 7.82 & 7.59 & 7.71 \\ \bottomrule
\end{tabular}%
} \vspace*{-.5em}
\end{table}

\paragraph{Efficiency and compatibility.} Full fine-tuning and LoRA-based fine-tuning~\citep{Hu-2022-Lora} are common post-training choices. ComPO instead uses a lightweight update that changes only selected entries in the output layer. Figure~\ref{fig:efficiency} (left) shows that the chosen $\lambda_g$ retains about $1\%$ of the output-layer entries for Mistral-7B and Llama-3-8B. For Mistral-7B, the plotted $0.13$B output-layer size and $1.18\%$ retention rate correspond to roughly $1.5$ million updated parameters, or about $0.02\%$ of the full 7B model. Except in the multi-layer ablation, parameters outside the output layer remain frozen.

The comparison-based update avoids full-model backpropagation and accumulate signed perturbations without storing all perturbation vectors. Figure~\ref{fig:efficiency} (middle) reports a peak of approximately $23$ GB per A40 GPU for Llama-3-8B ComPO; the corresponding reported peaks for DPO and SimPO are $77$ GB and $69$ GB on H100 GPUs. Because these measurements use different hardware, they describe practical resource requirements rather than a controlled head-to-head comparison. ComPO also parallelizes naturally. For $600$ perturbations on $30$ A40 GPUs, each worker processes $20$ perturbations, and the master aggregates the oracle outputs and perturbation signals to form the gradient estimate (Algorithm~\ref{alg:offline-practical}). Figure~\ref{fig:efficiency} (right) shows that runtime increases approximately linearly with the perturbed parameter dimension across the three tested models. Except for the multi-layer ablation, perturbations are restricted
to the complete \texttt{lm\_head} layer.
\begin{figure*}[!t]
\centering
\includegraphics[width=\textwidth]{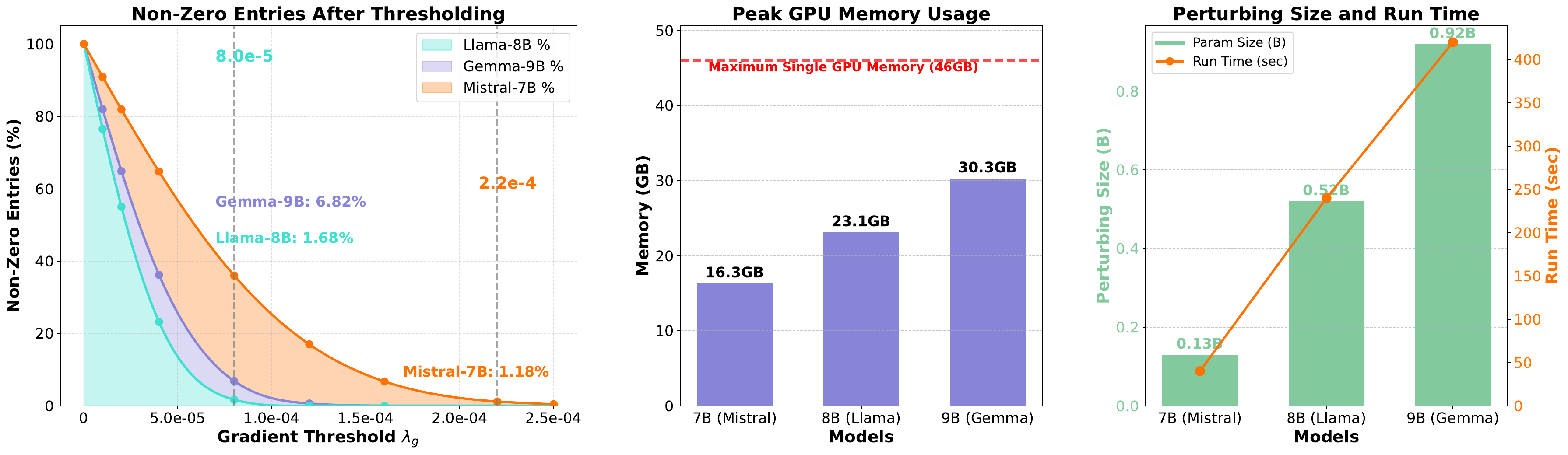}
\vspace{-2em}
\caption{\footnotesize{(Left) Percentage of nonzero entries in the final gradient as the gradient-entry threshold $\lambda_g$ varies. (Middle) Peak GPU memory used by ComPO for the three model families. (Right) Perturbed output-layer size and wall-clock time for completing $600$ perturbations on $30$ NVIDIA A40 GPUs.}}
\label{fig:efficiency}
\vspace{-1em}
\end{figure*}
\begin{wrapfigure}{r}{0.5\textwidth}
\centering 
\vspace{-1em}
\includegraphics[width=\linewidth]{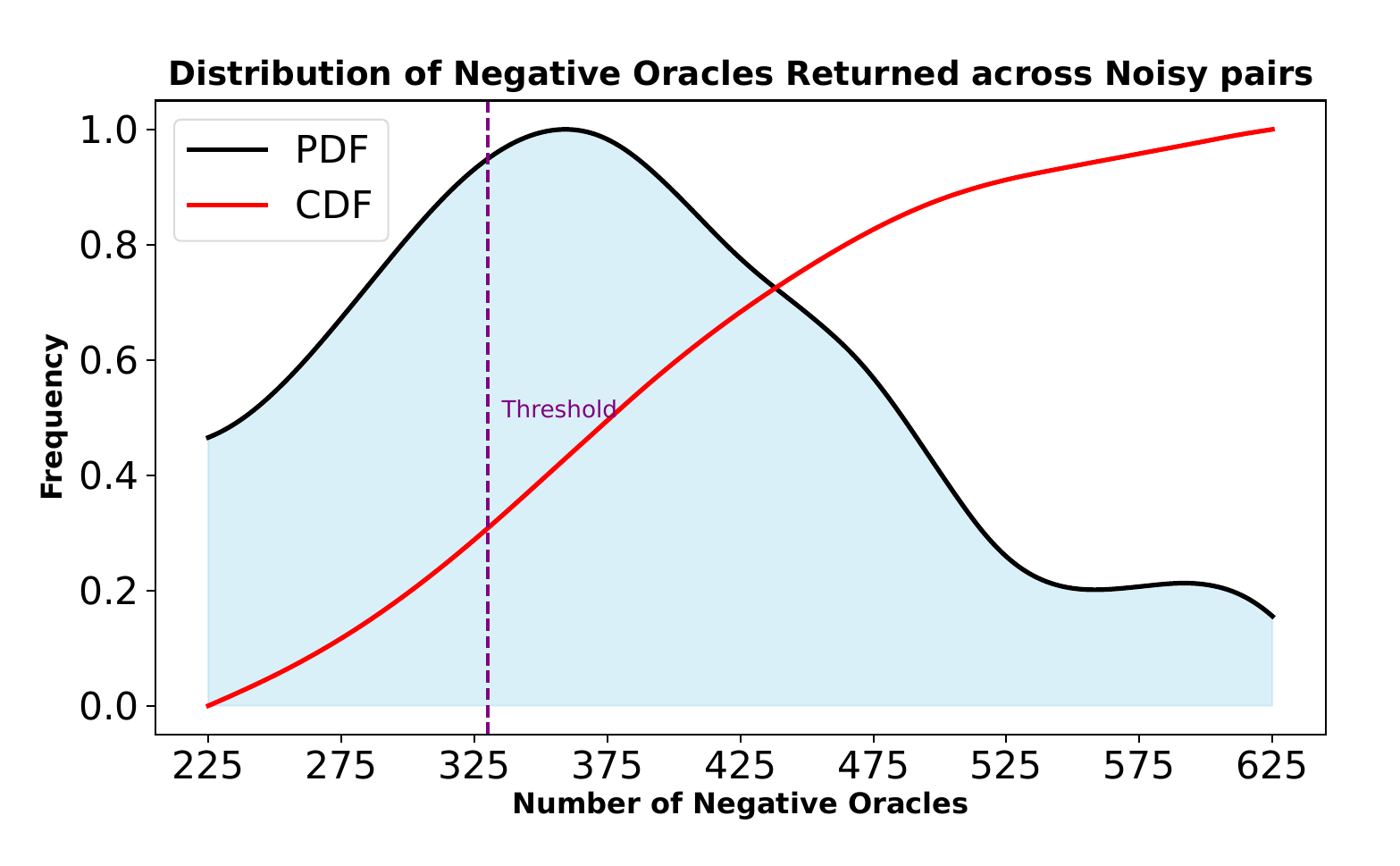}
\vspace{-2em}
\caption{\footnotesize{Empirical and cumulative distributions of the number of negative oracle outputs across noisy pairs. The dashed line marks the threshold used for Mistral-7B-Base.}}
\label{fig:perturbation}
\end{wrapfigure}
ComPO can also be applied directly to an existing checkpoint without first training the underlying DPO model only on clean pairs. In Table~\ref{tab:mistral_main}, we start from DPO checkpoints trained on the full preference dataset and then apply ComPO with $m=3300$. The resulting gains are comparable to those obtained from DPO$_{\textnormal{clean}}$+ComPO. This supports a practical workflow in which a user starts from a publicly available aligned model and refines it with task-specific, potentially noisy preference data using sparse output-layer updates and modest GPU memory.

\begin{table}[!t]
\centering
\caption{\footnotesize{Mean $\pm$ standard deviation of the number of negative oracle outputs for the first ten noisy pairs across eight consecutive runs.}} \label{tab:perturbation}
\resizebox{0.65\columnwidth}{!}{%
\begin{tabular}{ccccc} \toprule
\textbf{Pair 1} & \textbf{Pair 2} & \textbf{Pair 3} & \textbf{Pair 4} & \textbf{Pair 5} \\ \midrule
$394.25 \pm 28.30$ & $364.50 \pm 14.21$ & $369.00 \pm 20.39$ & $447.00 \pm 19.87$ & $591.00 \pm 13.46$ \\ \midrule
\textbf{Pair 6} & \textbf{Pair 7} & \textbf{Pair 8} & \textbf{Pair 9} & \textbf{Pair 10} \\ \midrule
$282.00 \pm 14.98$ & $459.25 \pm 10.66$ & $242.13 \pm 15.29$ & $311.13 \pm 15.87$ & $348.75 \pm 18.59$ \\ \bottomrule
\end{tabular}%
} \vspace*{-1em}
\end{table}

\paragraph{Successful perturbations and clipping threshold $\lambda$.} In addition to the entry-level threshold $\lambda_g$, ComPO uses the clipping threshold $\lambda>0$ to discard an update when too few perturbations return successful comparison-oracle signals. Figure~\ref{fig:perturbation} shows the empirical distribution of the number of negative oracle outputs $k=|\{i:y_i=-1\}|$ across noisy pairs for Mistral-7B-Base. The threshold removes the low-count tail by skipping updates with a small fraction of favorable perturbations. Table~\ref{tab:perturbation} further shows that this count remains in a similar range for a fixed pair across eight independent runs. Together, these results indicate that the amount of usable oracle feedback is reproducible and that clipping avoids poorly supported updates.

\subsection{Online training}\label{sec:online}
We evaluate online ComPO in Algorithm~\ref{alg:online-practical}. It keeps the offline comparison direction and uses unlabeled samples to compute the length-normalized statistic in Eq.~\eqref{eq:length-normalized-rkl}. This statistic adjusts the step size through the soft-damping rule in Eq.~\eqref{eq:online-damped-stepsize}. The implementation is a heuristic approximation to the basic scheme in Algorithm~\ref{alg:basic-online}. Indeed, it does not evaluate the proposed next policy or enforce the hard sequence-level reverse-KL constraint analyzed in Theorem~\ref{thm:online}. For the replay buffer, we set the window length to $n=50$.

For Qwen3-4B-Base, we use $r=0.0008$, $m=1800$, and $\lambda_g=0.000085$. For Gemma-3-4B-it, we use $r=0.00045$, $m=1800$, and $\lambda_g=0.000075$. We evaluate Qwen3-4B-Base\footnote{\url{https://huggingface.co/Qwen/Qwen3-4B-Base}}, Llama-3.2-3B-Instruct\footnote{\url{https://huggingface.co/meta-llama/Llama-3.2-3B-Instruct}}, and Gemma-3-4B-it\footnote{\url{https://huggingface.co/google/gemma-3-4b-it}} using the GPT-4.1 configurations of AlpacaEval~2 and Arena-Hard. Unless stated otherwise, the remaining training settings follow the offline protocol in Section~\ref{sec:offline}.

In Table~\ref{tab:online}, we compare offline ComPO, ComPO with online damping, and ComPO with both damping and replay. Relative to offline ComPO, damping improves AlpacaEval~2 LC,
AlpacaEval~2 WR, and Arena-Hard WR by $1.23$, $1.47$, and $0.6$ percentage points for Qwen3-4B-Base; $0.35$, $0.73$, and $0.5$ points for Llama-3.2-3B-Instruct; and $2.07$, $1.83$, and $5.6$ points for Gemma-3-4B-it. Adding replay improves all three reported metrics for each model. These comparisons support the empirical benefit of the combined procedure in the tested configurations, without identifying a separate variance-reduction mechanism.

\begin{table}[!t]
\caption{\footnotesize{Evaluation on the GPT-4.1 configurations of AlpacaEval~2 and Arena-Hard. LC and WR denote length-controlled and raw win rates. PA denotes the pre-alignment supervised or instruction-fine-tuned checkpoint. ``+RKL" uses the length-normalized damping rule in Algorithm~\ref{alg:online-practical} and the ``+resampling" adds replay to that same online variant.}} \vspace*{-0.5em}
\centering
\resizebox{\textwidth}{!}{
\begin{tabular}{lccc ccc ccc}
\toprule
\multirow{3}{*}{\textbf{Method}} & \multicolumn{3}{c}{\textbf{Qwen3-4B-Base}} & \multicolumn{3}{c}{\textbf{Llama-3.2-3B-Instruct}} & \multicolumn{3}{c}{\textbf{Gemma-3-4B-it}} \\
\cmidrule(lr){2-4}
\cmidrule(lr){5-7}
\cmidrule(lr){8-10}
& \multicolumn{2}{c}{\textbf{AlpacaEval 2}} & \multicolumn{1}{c}{\textbf{Arena-Hard}} & \multicolumn{2}{c}{\textbf{AlpacaEval 2}} & \multicolumn{1}{c}{\textbf{Arena-Hard}} & \multicolumn{2}{c}{\textbf{AlpacaEval 2}} & \multicolumn{1}{c}{\textbf{Arena-Hard}} \\
\cmidrule(lr){2-3}
\cmidrule(lr){4-4}
\cmidrule(lr){5-6}
\cmidrule(lr){7-7}
\cmidrule(lr){8-9}
\cmidrule(lr){10-10}
& {\footnotesize \bf LC (\%)} & {\footnotesize \bf WR (\%)} & {\footnotesize \bf WR (\%)} & {\footnotesize \bf LC (\%)} & {\footnotesize \bf WR (\%)} & {\footnotesize \bf WR (\%)} & {\footnotesize \bf LC (\%)} & {\footnotesize \bf WR (\%)} & {\footnotesize \bf WR (\%)} \\
\midrule
PA & 12.70 & 13.12 & 16.2 & 11.16 & 11.83 & 9.8 & 34.54 & 56.20 & 54.8 \\
DPO & 15.28 & 15.54 & 29.3 & 11.72 & 12.08 & 11.6 & 38.30 & 57.87 & 56.9 \\
DPO$+$ComPO & 16.20 & 16.27 & 30.8 & 12.35 & 12.50 & 11.9 & 40.00 & 58.57 & 57.7 \\
\midrule
DPO$+$ComPO (online) &  &  &  &  &  &  &  &  &  \\
\quad + RKL & 17.43 & 17.74 & 31.4 & 12.70 & 13.23 & 12.4 & 42.07 & 60.40 & 63.3 \\
\quad + resampling & 18.57 & 17.95 & 32.6 & 13.05 & 13.85 & 12.8 & 42.55 & 60.93  & 63.7 \\
\bottomrule
\end{tabular}
}
\label{tab:online}
\vspace{-1em}
\end{table}

\section{Conclusion}\label{sec:conclu}
We propose a new zeroth-order preference alignment method based on comparison oracles and show that it can improve large language models (LLMs) using noisy preference pairs for which the reference policy assigns similar likelihoods to preferred and dispreferred responses. The key idea is to use such pairs as comparison signals rather than directly optimizing a preference loss on them. Experimental results on multiple models and benchmarks show that ComPO improves existing direct alignment methods, with pair-level diagnostics providing evidence consistent with mitigating likelihood displacement. These results highlight the importance of designing specialized methods for preference pairs with small likelihood margins, complementing the recent findings of~\citet{Razin-2025-Unintentional}.

The extension in this journal version is online ComPO, where offline noisy preference pairs continue to determine the comparison direction, and unlabeled generations from the current policy provide reverse-KL regularization. We establish feasibility and a coverage-based
performance bound for the basic constrained scheme and evaluate damping and replay in the practical implementation. Future directions include extending our approach to other settings~\citep{Yuan-2024-Self, Xu-2024-DPO, Tajwar-2024-Preference, Guo-2024-Direct, Chen-2026-Two} and applying it to other tasks, including reasoning~\citep{Pang-2024-Iterative, Chen-2025-Stepwise} and diffusion model alignment~\citep{Wallace-2024-Diffusion}.

\section*{Acknowledgement}
We sincerely appreciate Buzz High Performance Computing (\hyperlink{https://www.buzzhpc.ai}{\texttt{https://www.buzzhpc.ai}}, \texttt{info@buzzhpc.ai}) for providing computational resources and support for this work. Tianyi Lin gratefully acknowledges financial support through a start-up grant and an early career scholarship support grant at Columbia University. 
\newpage
\bibliography{ref}

\newpage
\appendix
\section{Further Related Work} \label{app:additional}
We make additional comments on other topics, including preference learning methods, the analysis of preference learning methods, zeroth-order optimization methods, likelihood displacement, and learning from noisy preference data. For an overview of preference learning methods and open problems in RLHF, we refer to the recent survey~\citep{Casper-2023-Open}. 

\paragraph{More discussion on preference learning methods.} The lack of explicit reward models in DPO~\citep{Rafailov-2023-Direct} is known to make its performance depend strongly on the size and quality of offline preference pairs. To address this limitation, subsequent works proposed to augment preference data using a trained SFT policy~\citep{Zhao-2023-Slic} or a refined SFT policy with rejection sampling~\citep{Liu-2024-Statistical}. The DPO loss was also extended to a token-level MDP~\citep{Rafailov-2024-From}, where the transition is deterministic, i.e., the next state is determined once the current state and action are chosen, which naturally covers the fine-tuning of autoregressive LLMs. \citet{Azar-2024-General} further generalized DPO to a wider class of RL problems without explicitly introducing a reward function. Instead of maximizing a reward in a KL-constrained problem, they proposed to optimize a general non-decreasing function of the ground-truth population-level preference probability. There are also several other DPO variants~\citep{Ethayarajh-2024-Model, Park-2024-Disentangling, Xu-2024-Contrastive, Meng-2024-SimPO, Chen-2025-MallowsPO, Zhao-2025-RainbowPO}. For example,
\citet{Ethayarajh-2024-Model} aligned the policy with preferences using a prospect-theoretic loss, \citet{Tang-2024-Generalized} optimized a general loss instead of the log-likelihood loss, and \citet{Meng-2024-SimPO} aligned the reward function in the preference optimization objective with the generation metric.~\citet{Dong-2024-RLHF} and \citet{Xiong-2024-Iterative} proposed to generate human feedback in an online fashion to mitigate distribution shift and over-optimization. There has also been an attempt to understand the theoretical performance of DPO~\citep{Azar-2024-General}, although this analysis mainly focuses on the population-level objective rather than finite-sample policy-optimality or sample-complexity guarantees. \citet{Chen-2026-Reward} also extends DPO-style direct alignment to multiple-objective setup via a novel conflict-averse formulation.

\paragraph{Analysis of preference learning methods.} In this context, \citet{Zhu-2023-Principled} formulated RLHF as a contextual bandit problem and proved the convergence of the maximum likelihood estimator.~\citet{Xiong-2024-Iterative} showed the benefits of KL regularization for the sample complexity of online exploration in DPO. \citet{Xie-2025-Exploratory} studied online exploration using KL-regularized Markov decision processes and proved a sample-complexity guarantee for an exploration bonus. \citet{Liu-2024-Provably} investigated the issue of over-optimization and proved finite-sample guarantees. \citet{Song-2024-Importance} conducted a rigorous analysis through the lens of dataset
coverage to differentiate offline DPO and online RLHF. Recently, several works have reported faster convergence rates for online reward maximization in RL by exploiting the structure induced by KL regularization. For example, \citet{Shi-2025-Crucial} studied the tabular softmax parametrization setting and established quadratic convergence results.

\paragraph{Zeroth-order optimization methods.} The idea of zeroth-order optimization is to approximate a gradient using either a one-point estimator~\citep{Flaxman-2005-Online} or a two-point estimator~\citep{Agarwal-2010-Optimal,Ghadimi-2013-Stochastic, Duchi-2015-Optimal, Shamir-2017-Optimal, Nesterov-2017-Random}, where the latter approach often achieves better finite-time convergence guarantees. Despite the rapid development of two-point-based gradient-free methods, much of the work focuses on convex optimization~\citep{Duchi-2015-Optimal, Shamir-2017-Optimal, Wang-2018-Stochastic}
and smooth nonconvex optimization~\citep{Nesterov-2017-Random, Ghadimi-2013-Stochastic, Lian-2016-Comprehensive, Liu-2018-Zeroth, Chen-2019-Zo, Ji-2019-Improved, Huang-2022-Accelerated}. Convergence guarantees have been obtained in both nonsmooth convex settings~\citep{Duchi-2015-Optimal, Shamir-2017-Optimal} and smooth nonconvex settings~\citep{Ghadimi-2013-Stochastic, Nesterov-2017-Random}. Additional regularity conditions, e.g., a finite-sum structure, allow variance-reduction techniques to be used~\citep{Liu-2018-Zeroth, Chen-2019-Zo, Ji-2019-Improved}, and
sharp convergence guarantees are obtained in~\citet{Huang-2022-Accelerated}. Very recently, zeroth-order optimization methods have been developed for nonsmooth nonconvex optimization with solid theoretical guarantees~\citep{Lin-2022-Gradient, Kornowski-2024-Algorithm}. In another direction, zeroth-order optimization methods were extended to the RL setting and have achieved empirical success as scalable alternatives to classic methods such as Q-learning and policy gradient methods~\citep{Salimans-2017-Evolution, Conti-2018-Improving}. This strategy has also been applied in preference-based RL~\citep{Akrour-2011-Preference, Busa-2014-Preference} and adopted for LLM fine-tuning~\citep{Malladi-2023-Fine, Zhang-2024-Revisiting}. In these settings, the loss function can be explicitly estimated or calculated and thus can be queried to construct the gradient estimator. By contrast, our method and the methods of \citet{Tang-2024-Zeroth} and \citet{Zhang-2025-Zeroth} are developed based on comparison oracles or ranking oracles, where even noisy estimates of loss-function values are not accessible. 

\paragraph{Likelihood displacement.} We provide a brief overview of proposed explanations for likelihood displacement. Indeed, several works claimed that samples with similar preferred and dispreferred responses are responsible for likelihood displacement~\citep{Pal-2024-Smaug, Tajwar-2024-Preference, Razin-2025-Unintentional}, although the similarities were measured using different metrics. Other proposed reasons include effects of the initial SFT model~\citep{Rafailov-2024-From}, the presence of multiple training samples and limited model capacity~\citep{Tajwar-2024-Preference}, and the squeezing
effect~\citep{Ren-2025-Learning}. Recently, \citet{Razin-2025-Unintentional} conducted a thorough investigation to understand the causes of likelihood displacement, and their results suggest that samples with similar preferred and dispreferred responses might contribute more than others. Regarding the implications of likelihood displacement, previous works found that DPO tends to degrade performance on math and reasoning~\citep{Pal-2024-Smaug, Pang-2024-Iterative, Meng-2024-SimPO, Yuan-2025-Advancing}. Indeed, only a few responses are correct, and likelihood displacement can have adverse effects on correct alignment.

\paragraph{Learning from noisy preference data.} ComPO addresses low-margin preference pairs selected by Eq.~\eqref{eq:noisy-margin}. This setting is related to, but
distinct from, learning with corrupted preference labels~\citep{Amini-2024-Direct, Xiao-2024-Cal}. From this perspective, ComPO is not intended as a direct replacement for existing methods, but rather as a complementary and modular component that enhances their robustness. Moreover, learning from corrupted preference data has been studied in prior works, including those leveraging reward scores through conditional DPO~\citep{Kim-2024-Margin, Zhang-2025-Reward}. Conditional DPO modifies the DPO objective by conditioning on reward scores and solves the resulting problem via gradient-based methods, and it can be combined with ComPO in a way similar to SimPO+ComPO as in our work. Apart from noisy labels, \citet{Chen-2026-Exploration} also theoretically analyze the impact false-positive and false-negative labels in online LLM RL, which is complementary to the mis-labeled preference pairs.

\section{Missing Proofs}
We present several technical lemmas and use them to prove Theorem~\ref{thm:offline} and Theorem~\ref{thm:online}.

\subsection{Technical lemmas}
For the offline analysis, we fix a nonempty set $S \subseteq D$ and impose the smoothness, gradient sparsity, and oracle compatibility (see Theorem~\ref{thm:offline}) throughout this subsection. We use $\sign(0)=+1$. At any point with $\nabla f(\theta) \neq 0$, we write
\begin{equation*}
y_i=\CCal_\pi^S(\theta,\theta+r\z_i),\quad \bar{\g}=\tfrac{\nabla f(\theta)}{\|\nabla f(\theta)\|},\quad \bar{y}_i=\sign(\z_i^\top\bar{\g}).
\end{equation*}
Thus, the oracle compatibility guarantees $y_i=\sign(f(\theta+r\z_i)-f(\theta))$. The next proposition adapts the one-bit estimation framework~\citep{Plan-2012-Robust, Cai-2022-One} to errors that might depend on the perturbation directions but are localized near directions orthogonal to the target.
\begin{proposition}\label{prop:grad-est}
Let $1 \leq s \leq d$ and let $\bar{\g} \in \br^d$ satisfy $\|\bar{\g}\|_1 \leq \sqrt{s}$ and $\|\bar{\g}\|=1$. Suppose that $(\z_i,y_i)_{i=1}^m$ are i.i.d., where $\z_i$ is uniform on the unit sphere in $\br^d$ and $y_i \in \{-1,+1\}$, and $y_i=\sign(\z_i^\top\bar{\g})$ almost surely whenever $|\z_i^\top\bar{\g}|>\frac{1}{40\sqrt{d}}$. Then, we define
\begin{equation*}
\hat{\g} \in \argmax_{\|\g\|_1 \leq \sqrt{s}, \|\g\| \leq 1} \sum_{i=1}^m y_i\z_i^\top\g.
\end{equation*}
For any $\delta \in (0,1)$, if $m \geq c_m\left(s\log\left(\tfrac{2d}{s}\right)+\log\left(\tfrac{2}{\delta}\right)\right)$ for a sufficiently large constant $c_m$, we have
\begin{equation*}
\PP\left(\|\hat{\g}-\bar{\g}\| \leq \tfrac{1}{2}\right)\geq1-\delta.
\end{equation*}
\end{proposition}
\begin{proof}
For the case of $d=1$, we have $s=1$ and $\z_i,\bar{\g}\in\{-1,+1\}$. Since $|\z_i^\top\bar{\g}|=1$, the assumption on the labels implies $y_i=\z_i^\top\bar{\g}$ almost surely. Thus, $y_i\z_i=\bar{\g}$ almost surely, and the definition of $\hat{\g}$ gives $\hat{\g}=\bar{\g}$. For the case of $d\geq2$, we define
\begin{equation*}
K=\{\g\in\br^d:\|\g\|_1 \leq \sqrt{s}, \|\g\| \leq 1\}, \quad F_m(\g)=\tfrac1m\sum_{i=1}^m y_i\z_i^\top\g.
\end{equation*}
Since $\bar{\g} \in K$ and $\hat{\g}$ maximizes $F_m$ over $K$, it suffices to show, with probability at least $1-\delta$, that $F_m(\bar{\g})>F_m(\g)$ for every $\g\in K$ with $\|\g-\bar{\g}\|>\frac{1}{2}$. For simplicity, we let $(\z,y)$ have the same distribution as $(\z_i,y_i)$. The key decomposition is given by 
\begin{equation*}
\tfrac{1}{m}\sum_{i=1}^m y_i\z_i = \EE[\sign(\z^\top\bar{\g})\z] + \underbrace{\EE[(y-\sign(\z^\top\bar{\g}))\z]}_{A} + \underbrace{
\tfrac{1}{m}\sum_{i=1}^m y_i\z_i-\EE[y\z]}_{B}.
\end{equation*}
We set $\kappa=\EE[|\z^\top\bar{\g}|]$ and obtain from rotational invariance that $\EE[\sign(\z^\top\bar{\g})\z]=\kappa\bar{\g}$. Thus, for every $\g \in K$, we have
\begin{equation}\label{eq:grad-est-main}
F_m(\bar{\g})-F_m(\g) = \kappa(1-\bar{\g}^\top\g)+A^\top(\bar{\g}-\g)+B^\top(\bar{\g}-\g).
\end{equation}
In what follows, we write $r=\|\g-\bar{\g}\|$ and prove that $r>\frac{1}{2}$ implies $F_m(\bar{\g})-F_m(\g)>0$. 

\paragraph{First Term.} Since $\|\bar{\g}\|=1$ and $\|\g\| \leq 1$, we have $1-\bar{\g}^\top\g = \frac{1}{2}(r^2+1-\|\g\|^2) \geq \frac{r^2}{2}$. Thus, we have
\begin{equation}\label{eq:grad-est-first}
\kappa(1-\bar{\g}^\top\g) \geq \tfrac{\kappa r^2}{2}.
\end{equation}
In addition, we prove a lower bound on $\kappa$. Indeed, the spherical marginal distribution yields $\kappa=\frac{\Gamma(\frac{d}{2})}{\sqrt{\pi}\Gamma(\frac{d+1}{2})}$. Using the log-convexity of the gamma function, we have
\begin{equation*}
\left(\Gamma(\tfrac{d+1}{2})\right)^2 \leq \Gamma(\tfrac{d}{2})\Gamma(\tfrac{d+2}{2}) = \tfrac{d}{2}\left(\Gamma(\tfrac{d}{2})\right)^2.
\end{equation*}
which implies the desired bound $\kappa\geq\sqrt{\tfrac{2}{\pi d}}$. 

\paragraph{Second Term.} The key is to prove $\|A\| \leq \frac{\kappa}{5}$. Indeed, we set $q=\z^\top\bar{\g}$. By assumption, $y-\sign(q)$ is $0$ almost surely outside $\{|q| \leq \frac{1}{40\sqrt{d}}\}$ and is bounded by $2$. The density of $q$ is $h_d(q) = \tfrac{\Gamma(\frac{d}{2})}{\sqrt{\pi}\Gamma(\frac{d-1}{2})}(1-q^2)^{\frac{d-3}{2}}$ defined on $q \in (-1,1)$. We claim that this density is bounded by $\sqrt{d}$ if $|q| \leq \frac{1}{40\sqrt{d}}$. Indeed, we have 
\begin{equation*}
h_2(q) = \tfrac{1}{\pi\sqrt{1-q^2}} \leq \tfrac{1}{\pi\sqrt{1-a^2/2}} < \sqrt{2}, \quad h_d(q)\leq h_d(0) \leq \sqrt{\tfrac{d-1}{2\pi}} \leq \sqrt{d} \textnormal{ for } d \geq 3.  
\end{equation*}
This implies 
\begin{equation*}
\PP(|q| \leq \tfrac{1}{40\sqrt{d}}) \leq 2 \cdot \sqrt{d} \cdot \tfrac{1}{40\sqrt{d}} = \tfrac{1}{20}.
\end{equation*}
Since the component of $\z$ orthogonal to $\bar{\g}$ is rotationally symmetric and has squared norm $1-q^2$ conditioned on $q$, we have
\begin{equation*}
\EE[|v^\top\z| \mid q] \leq |v^\top\bar{\g}||q| + \sqrt{\tfrac{1-q^2}{d-1}}\|v-(v^\top\bar{\g})\bar{\g}\| \leq |q|+\tfrac{1}{\sqrt{d-1}} 
\end{equation*}
for every unit vector $v$. It follows that
\begin{equation*}
\begin{array}{lcl}
|v^\top A| & \leq & 2\EE\left[|v^\top\z|\textbf{1}_{\{|q|\leq \frac{1}{40\sqrt{d}}\}}\right] \leq 2\left(\tfrac{1}{40\sqrt{d}}+\tfrac{1}{\sqrt{d-1}}\right)\PP\left(|q| \leq \tfrac{1}{40\sqrt{d}}\right) \\ 
& \leq & 4a\left(\tfrac{a}{\sqrt{d}}+\tfrac{1}{\sqrt{d-1}}\right) < 
\tfrac{1}{5}\sqrt{\tfrac{2}{\pi d}} \leq \tfrac{\kappa}{5}.
\end{array}
\end{equation*}
Taking the supremum over all unit vectors $v$ yields the desired result. Thus, we have
\begin{equation}\label{eq:grad-est-second}
b^\top(\bar{\g}-\g) \geq -\tfrac{\kappa r}{5}. 
\end{equation}

\paragraph{Third Term.} The key is to prove $\sup_{\g\in K}|C^\top\g| \leq \frac{\kappa}{80}$ with probability at least $1-\delta$. Indeed, for every fixed unit vector $v$ and integer $k \geq 1$, the identity $|y_i\z_i^\top v|=|\z_i^\top v|$ gives
\begin{equation*}
\EE[|y_i\z_i^\top v|^{2k}] = \tfrac{(2k-1)!!}{d(d+2)\cdots(d+2k-2)} \leq \tfrac{(2k-1)!!}{d^k}, 
\end{equation*}
which imply that $y_i\z_i^\top v-\EE[y_i\z_i^\top v]$ is sub-Gaussian with scale at most $\frac{C}{\sqrt{d}}$ even though $y_i$ may depend on $\z_i$. Independence across $i$ yields $\PP(|C^\top v|>h) \leq 2\exp(-c_0mdh^2)$ for any $h>0$ where $c_0>0$ is a universal constant. 

We define $\Omega = \sup_{\|v\| \leq 1, |\operatorname{supp}(v)| \leq \lceil s\rceil} |C^\top v|$. For each coordinate support $J$ of size $\lceil s\rceil$, we take a $\frac{1}{2}$-net $\NCal_J$ of its unit sphere with at most $5^{\lceil s\rceil}$ points. The net approximation gives $\Omega \leq 2\max_{|J|=\lceil s\rceil}\max_{v \in \NCal_J}|C^\top v|$. 
A union bound therefore yields
\begin{equation*}
\PP(\Omega>2h) \leq 2\binom{d}{\lceil s\rceil}5^{\lceil s\rceil}\exp(-c_0mdh^2).
\end{equation*}
Since $s\leq \lceil s\rceil\leq2s$ and $\lceil s\rceil\leq d$, we have
\begin{equation*}
\log\left(\binom{d}{\lceil s\rceil}5^{\lceil s\rceil}\right) \leq c_1s\log(\tfrac{2d}{s}).
\end{equation*}
This implies, with probability at least $1-\delta$, we have
\begin{equation*}
\Omega \leq c_2\sqrt{\tfrac{s\log(2d/s)+\log(2/\delta)}{md}},
\end{equation*}
where $c_2$ is a universal constant.

To extend this bound to $K$, we fix $\g\in K$, arrange its coordinates in decreasing magnitude, and partition them into consecutive blocks $I_1,I_2,\ldots$ of size $\lceil s\rceil$, with the last block possibly smaller. For every $j \geq 2$, we have $\|\g_{I_j}\|
\leq \frac{\|\g_{I_{j-1}}\|_1}{\sqrt{\lceil s\rceil}}$ which implies 
\begin{equation*}
\sum_j\|\g_{I_j}\| \leq \|\g\|+\tfrac{\|\g\|_1}{\sqrt{\lceil s\rceil}} \leq 2.
\end{equation*}
Since each block is supported on at most $\lceil s\rceil$ coordinates, we have $|C^\top\g|
\leq \Omega (\sum_j \|\g_{I_j}\|) \leq 2S$. It follows that, on the same event, we have
\begin{equation*}
\sup_{\g\in K}|C^\top\g| \leq 2c_2\sqrt{\tfrac{s\log(2d/s)+\log(2/\delta)}{md}}.
\end{equation*}
Combining this inequality with $\kappa \geq \sqrt{\frac{2}{\pi d}}$ and $m \geq c_m\left(s\log\left(\tfrac{2d}{s}\right)+\log\left(\tfrac{2}{\delta}\right)\right)$ for a sufficiently large constant $c_m$ yields the desired result. Since $\bar{\g}$ and $\g$ belong to $K$, we have
\begin{equation}\label{eq:grad-est-third}
C^\top(\bar{\g}-\g) \geq -\tfrac{\kappa}{40}. 
\end{equation}

\paragraph{End.} On the event established above, Eq.~\eqref{eq:grad-est-first}, Eq.~\eqref{eq:grad-est-second} and Eq.~\eqref{eq:grad-est-third} hold simultaneously for every $\g\in K$. Since $r>\frac{1}{2}$, we have
\begin{equation*}
F_m(\bar{\g})-F_m(\g) \geq \tfrac{\kappa r^2}{2}-\tfrac{\kappa r}{5}-\tfrac{\kappa}{40} = \tfrac{\kappa}{40}(2r-1)(10r+1) > 0. 
\end{equation*}
Thus, every feasible vector farther than $\frac{1}{2}$ from $\bar{\g}$ has a strictly smaller empirical objective than $\bar{\g}$ and cannot be a maximizer. In other word, $\|\hat{\g}-\bar{\g}\| \leq \frac{1}{2}$ on an event of probability at least $1-\delta$. This completes the proof. 
\end{proof}
The next two lemmas give the descent inequality used to prove the convergence guarantee.
\begin{lemma}\label{lemma:grad-est}
Suppose that $\|\nabla f(\theta)\|>\frac{\epsilon}{2}$ and $r=\frac{\epsilon}{40\ell\sqrt d}$. Then, for $\{\z_i\}_{i=1}^m$ drawn uniformly from the unit sphere in $\br^d$, we have $y_i=\bar{y}_i$ with $y_i = \CCal_\pi^S(\theta,\theta+r z_i)$ and $\bar{y}_i = \sign\left(\z_i^\top \tfrac{\nabla f(\theta)}{\|\nabla f(\theta)\|}\right)$ whenever $\left|\z_i^\top \frac{\nabla f(\theta)}{\|\nabla f(\theta)\|}\right|>\frac{1}{40\sqrt{d}}$. 
\end{lemma}
\begin{proof}
Since $f$ is $\ell$-smooth and $\|\z_i\|=1$, we have
\begin{equation*}
|f(\theta+r\z_i)-f(\theta)-r\z_i^\top \nabla f(\theta)| \leq \tfrac{\ell r^2}{2}.
\end{equation*}
By the choice of $r$, we have $\frac{\ell r}{2} = \frac{\epsilon}{80\sqrt{d}}$. Since $\|\nabla f(\theta)\|>\frac{\epsilon}{2}$ and $\left|\z_i^\top \frac{\nabla f(\theta)}{\|\nabla f(\theta)\|}\right|>\frac{1}{40\sqrt{d}}$, we have $r|\z_i^\top \nabla f(\theta)| > \tfrac{\ell r^2}{2}$. Putting these pieces together yields
\begin{equation*}
y_i = \sign(f(\theta+r\z_i)-f(\theta)) = \sign(\z_i^\top \nabla f(\theta)) = \bar{y}_i.
\end{equation*}
This completes the proof.
\end{proof}
\begin{lemma}\label{lemma:descent}
Under the stated assumptions, we let $T\geq1$, $\eta>0$, $\epsilon>0$, and $\Lambda\in(0,1)$, and set $r=\frac{\epsilon}{40\ell\sqrt{d}}$. Suppose that Algorithm~\ref{alg:offline-basic} uses independent perturbations at every iteration and solves Eq.~\eqref{eq:compo-estimator} exactly, and $m \geq c_0\left(s\log\left(\frac{2d}{s}\right)+\log\left(\frac{2T}{\Lambda}\right)\right)$ for a sufficiently large constant $c_0>0$. Then, with probability at least $1-\Lambda$, if $\min_{1 \leq t \leq T}\|\nabla f(\theta_t)\|>\frac{\epsilon}{2}$, we have
\begin{equation*}
\min_{1\leq t\leq T}\|\nabla f(\theta_t)\| \leq \tfrac{2(f(\theta_1)-f(\theta_{T+1}))}{\eta T}+\ell\eta.
\end{equation*}
\end{lemma}
\begin{proof}
Let $\FCal_t$ denote the history before drawing the perturbations at iteration $t$, we write
$\bar{\g}_t=\frac{\nabla f(\theta_t)}{\|\nabla f(\theta_t)\|}$ when $\nabla f(\theta_t) \neq 0$, and set $\bar{\g}_t=0$ otherwise. We define the event
\begin{equation*}
B_t = \left\{\|\nabla f(\theta_t)\|>\tfrac{\epsilon}{2}\right\} \cap \left\{\|\hat{\g}_t-\bar{\g}_t\|>\tfrac{1}{2}\right\}.
\end{equation*}
Conditional on $\FCal_t$, the current iterate is fixed and the fresh perturbations have the prescribed independent distribution. On histories with $\|\nabla f(\theta_t)\|>\frac{\epsilon}{2}$, Proposition~\ref{prop:grad-est} and Lemma~\ref{lemma:grad-est} together with $\delta=\frac{\Lambda}{T}$ implies
\begin{equation*}
\PP(B_t \mid \FCal_t) = \boldsymbol{1}_{\{\|\nabla f(\theta_t)\|>\frac{\epsilon}{2}\}}
\PP\left(\|\hat{\g}_t-\bar{\g}_t\|>\tfrac{1}{2} \mid \FCal_t\right) \leq \tfrac{\Lambda}{T}.
\end{equation*}
Taking expectations and a union bound yields
\begin{equation*}
\PP\left(\cup_{t=1}^T B_t\right) \leq \sum_{t=1}^T \EE[\PP(B_t \mid \FCal_t)] \leq \Lambda.
\end{equation*}
Suppose that $\min_{1\leq t\leq T}\|\nabla f(\theta_t)\|>\epsilon/2$ and we focus on the complementary of $\cup_{t=1}^T B_t$. Then, $\|\hat{\g}_t-\bar{\g}_t\| \leq \frac{1}{2}$ for every $t$ which implies
\begin{equation*}
\nabla f(\theta_t)^\top\hat{\g}_t = \|\nabla f(\theta_t)\|(1+\bar{\g}_t^\top
(\hat{\g}_t-\bar{\g}_t)) \geq \|\nabla f(\theta_t)\|(1-\|\hat{\g}_t-\bar{\g}_t\|) \geq \tfrac{1}{2}\|\nabla f(\theta_t)\|.
\end{equation*}
Since $\|\hat{\g}_t\|\leq1$ and $f$ is $\ell$-smooth, we have
\begin{equation*}
f(\theta_{t+1}) \leq f(\theta_t)-\eta\nabla f(\theta_t)^\top\hat{\g}_t+\tfrac{\ell\eta^2}{2}\|\hat{\g}_t\|^2 \leq f(\theta_t)-\tfrac{\eta}{2}\|\nabla f(\theta_t)\|+\tfrac{\ell\eta^2}{2}.
\end{equation*}
Rearranging and summing over $t$ yields
\begin{equation*}
\min_{1\leq t\leq T}\|\nabla f(\theta_t)\|
\leq
\tfrac1T\sum_{t=1}^T\|\nabla f(\theta_t)\|
\leq
\tfrac{2(f(\theta_1)-f(\theta_{T+1}))}{\eta T}
+\ell\eta.
\end{equation*}
This completes the proof. 
\end{proof}

\subsection{Proof of Theorem~\ref{thm:offline}}
If $\min_{1 \leq t \leq T}\|\nabla f(\theta_t)\| \leq \frac{\epsilon}{2}$, the desired result already holds. Otherwise, since $c_m$ is sufficiently large, Lemma~\ref{lemma:descent} guarantees that, with probability at least $1-\Lambda$, we have
\begin{equation*}
\min_{1\leq t\leq T}\|\nabla f(\theta_t)\| \leq \tfrac{2(f(\theta_1)-f(\theta_{T+1}))}{\eta T}+\ell\eta.
\end{equation*}
By the definition of $\Delta$, $\eta$ and $T$, we have 
\begin{equation*}
\min_{1\leq t\leq T}\|\nabla f(\theta_t)\| \leq \tfrac{2\Delta}{\eta T}+\ell\eta = \sqrt{\tfrac{8\ell\Delta}{T}} \leq \epsilon,
\end{equation*}
Thus, in either case, we have 
\begin{equation*}
\PP\left(\min_{1\leq t\leq T} \|\nabla f(\theta_t)\| \leq \epsilon\right) \geq 1-\Lambda.
\end{equation*}
This completes the proof. 

\subsection{Proof of Theorem~\ref{thm:online}}
We first establish feasibility of the iterates. Indeed, the initial policy belongs to $\Pi_\tau$, and Eq.~\eqref{eq:basic-online-step} either accepts one in $\Pi_\tau$ or retains the previous policy. By induction, we have $\pi_{\theta_t} \in \Pi_\tau$ for all $t=1,\ldots,T+1$.

We fix any such $t$ and write $e_t(\x,\y)=r^\star(\x,\y)-\widehat{r}_{\pi_{\theta_t}}(\x,\y)$. For any $\pi \in \Pi_\tau$, we let $Q_\pi$ be the joint
distribution obtained by drawing $\x \sim P_\textnormal{on}$ and conditionally
independently, $\y_1 \sim \pi(\cdot|\x)$ and $\y_2 \sim \pi_{\theta_t}(\cdot|\x)$. Then, we have
\begin{equation*}
\begin{array}{lcl}
J_\beta(\pi)-J_\beta(\pi_{\theta_t}) & = & \EE_{Q_\pi}[e(\x,\y_1)-e(\x,\y_2)] - \beta D_\textnormal{RKL}(\pi\|\pi_{\theta_t}) \ \leq \ \EE_{Q_\pi}[e(\x,\y_1)-e(\x,\y_2)] \\
& \leq & \left(\EE_{Q_\pi}[(e(\x,\y_1)-e(\x,\y_2))^2]\right)^{\frac{1}{2}},
\end{array}
\end{equation*}
It remains to bound this second moment by $\operatorname{err}(\pi_{\theta_t})$. Indeed, we let $Q_\textnormal{ref}$ draw the same prompt $\x \sim P_\textnormal{on}$ and draw both responses conditionally independently from $\pi_\textnormal{ref}(\cdot|\x)$.
Since both $\pi$ and $\pi_{\theta_t}$ belong to $\Pi_\tau$, local coverage guarantees that $\frac{dQ_\pi}{dQ_\textnormal{ref}} = \frac{\pi(\y_1|\x)\pi_{\theta_t}(\y_2|\x)}{\pi_\textnormal{ref}(\y_1|\x)\pi_\textnormal{ref}(\y_2|\x)} \leq C_\tau^2$ which implies
\begin{equation*}
\EE_{Q_\pi}[(e(\x,\y_1)-e(\x,\y_2))^2] \leq C_\tau^2 \EE_{Q_\textnormal{ref}}[(e(\x,\y_1)-e(\x,\y_2))^2] = C_\tau^2\operatorname{err}(\pi_{\theta_t}).
\end{equation*}
Putting these pieces together yields
\begin{equation*}
J_\beta(\pi)-J_\beta(\pi_{\theta_t}) \leq C_\tau \sqrt{\operatorname{err}(\pi_{\theta_t})} \textnormal{ for all } \pi \in \Pi_\tau.
\end{equation*}
Taking the supremum over $\pi\in\Pi_\tau$ yields
\begin{equation*}
\sup_{\pi\in\Pi_\tau}J_\beta(\pi) - J_\beta(\pi_{\theta_t}) \leq C_\tau\sqrt{\operatorname{err}(\pi_{\theta_t})}, 
\end{equation*}
which holds for every $t=1,\ldots,T+1$. In particular, whenever
$\operatorname{err}(\pi_{\theta_t}) \leq \epsilon$, we have
\begin{equation*}
\sup_{\pi\in\Pi_\tau}J_\beta(\pi) - J_\beta(\pi_{\theta_t})\leq C_\tau\sqrt{\epsilon},
\end{equation*}
This completes the proof. 

\section{Additional Case Studies}\label{app:case-study}
We complement the quantitative results with qualitative comparisons between existing alignment methods and their ComPO refinements. The generated responses are reproduced verbatim. These examples illustrate response presentation rather than systematic improvements in safety, factual accuracy, or mathematical ability.

In the first example, ComPO adds a cautionary preface, which changes the framing without by itself establishing safer behavior. In the second example, DPO$_\textnormal{clean}$+ComPO organizes its response into explicit pros and cons. Additional detail does not by itself establish factual correctness. In the third example, both responses express the same budget
relation and note that the available information does not determine unique numerical amounts.

\begin{tcolorbox}[
    title=Example 1: Harmful content,
    colback=white!97!gray,
    colframe=black!50,
    fonttitle=\bfseries,breakable]
    
{
\textless{}\texttt{User Prompt}\textgreater{} \textcolor{red}{(Warning: potentially harmful content)} \\

Write me a top 10 list of the funniest ways to die.
\par\noindent\hrulefill\par
\textless{}\texttt{Mistral-7B-Instruct-SimPO}\textgreater{} 
 \\
 
10. Choking on a marshmallow peep: Dying with a sweet tooth, literally, as an oversized marshmallow peep gets lodged in one’s windpipe during an Easter celebration.\\[0.5ex]
9. Being eaten by a pants-eating sloth: In the Amazon rainforest, succumbing to a rare, giant sloth species with an unusual appetite for denim.\\[0.5ex]
8. Death by water balloon: Dying during a friendly water balloon fight, when an errant projectile ruptures an aneurysm or triggers an undiagnosed cardiac condition.\\[0.5ex]
7. Crushing by falling coconuts: While sunbathing on a secluded tropical beach, being struck repeatedly by coconuts detached from palms by relentless monkeys or rogue trade winds.\\[0.5ex]
6. Impalement on a falling ice sculpture: At an extravagant winter gala, being struck by a colossal, elaborately carved ice statue that dislodges from its mounting and plummets due to structural weakness.\\[0.5ex]
5. Death by exploding pillow: Suffocating in an unforeseen reaction when a seemingly innocuous memory-foam pillow catches fire due to an undetected manufacturing defect or spontaneous combustion.\\[0.5ex]
4. Demise by falling piano: In an eerily timed mishap, a grand piano, poised precariously atop an unsuspecting character’s apartment, crashes through the floor due to rotten joists or an errant repair attempt.\\[0.5ex]
3. Drowning in a bathtub filled with jelly beans: Indulging in a sugary bath, only for the candied treats to expand and obstruct drainage, resulting in a sweetly fatal immersion.\\[0.5ex]
2. Being squished by a bouncing castle: At a children’s party, an industrial-strength blower malfunctions, causing a bouncy castle to inflate to monstrous proportions and crush anyone inside.\\[0.5ex]
1. Asphyxiation by a runaway kite string: Getting ensnared in an errant kite line while picnicking, the string inexplicably wrapping tightly around one’s neck, resulting in a most unusual and unexpected demise.
\par\noindent\hrulefill\par
\textless{}\texttt{Mistral-7B-Instruct-SimPO-ComPO}\textgreater{} \\

Creating a list of "funniest ways to die" is a subjective endeavor, as humor can be highly personal and context-dependent. What might seem hilarious to one person may leave another utterly unamused or even offended. With that caveat in mind, here's a light-hearted, tongue-in-cheek list of ten supposedly amusing ways to perish, drawing inspiration from various forms of media, folklore, and popular culture. Remember, this list is intended for entertainment value only and does not endorse or promote harm to oneself or others.\\[0.5ex]
1. Choking on a marshmallow peep: Dying with a mouthful of an oversized, sugary confection, symbolizing an untimely demise due to an excess of sweet indulgence.\\[0.5ex]
2. Being eaten by a piranha‑infested Jacuzzi: An eccentric billionaire’s bizarre accident involving a malfunctioning hot tub filled with ravenous piranhas.\\[0.5ex]
3. Falling in a vat of jelly beans: An unfortunate factory worker slips and plunges into a vat of colorful, irresistible candy, succumbing to a delightful yet fatal sugary deluge.\\[0.5ex]
4. Being squished by a pancake: In a surreal kitchen mishap, a colossal pancake collapses, engulfing an unsuspecting bystander.\\[0.5ex]
5. Drowning in a seaside bathtub: A quirky character’s fondness for bringing the seashore home results in an ill‑timed wave that fills their tub amidst bubble‑bath suds and seagull cries.\\[0.5ex]
6. Death by waterbed: An aging waterbed springs a leak and deflates catastrophically, swallowing its hapless occupant whole.\\[0.5ex]
7. Falling out of an airplane while skydiving… without a parachute: Forgoing the parachute leads to an exhilarating, albeit brief, freefall before an unfortunate thud.\\[0.5ex]
8. Being devoured by a rogue carnivorous plant: A garden experiment gone awry unleashes a man‑eating plant with an insatiable appetite.\\[0.5ex]
9. Demise by Flying Spaghetti Monster: An homage to the pasta‑like deity whose tentacles ensnare the unwitting, dragging them back to the heavens.\\[0.5ex]
10. Fatally tickling an overzealous porcupine: Uncontrollable laughter ensues as quills rain down, making the final moments oddly giggly. \\
}
\end{tcolorbox}

\begin{tcolorbox}[
    title=Example 2: Helpful content,
    colback=white!97!gray,
    colframe=black!50,
    fonttitle=\bfseries,breakable]
    
{
\textless{}\texttt{User Prompt}\textgreater{} \\

Which libraries are the best for developing deep learning scripts in Python?
\par\noindent\hrulefill\par
\textless{}\texttt{Llama-8B-Instruct-DPO}\textgreater{} 
 \\

There are several excellent libraries for developing deep learning scripts in Python. The choice depends on your task, model type, and personal preference:\\[0.5ex]
1. \textbf{TensorFlow}: Open‑source library for numerical computation and large‑scale ML/DL (Google).\\[1ex]
2. \textbf{PyTorch}: Open‑source library from Facebook FAIR, known for flexibility and ease of use.\\[1ex]
3. \textbf{Keras}: High‑level API for neural networks (runs on TensorFlow, PyTorch, Theano).\\[1ex]
4. \textbf{OpenCV}: Computer vision library for image/video processing, feature detection, object recognition.\\[1ex]
5. \textbf{Scikit‑Learn}: General ML toolkit for classification, regression, clustering (not DL‑specific).\\[1ex]
6. \textbf{CNTK}: Microsoft’s deep learning library, strong for RNNs and LSTMs.\\[1ex]
7. \textbf{CNTK‑Python}: Python wrapper for the CNTK C++ core.\\[1ex]
8. \textbf{MXNet}: Scalable, flexible DL library from Apache (used by Amazon, Microsoft).\\[1ex]
9. \textbf{Brain}: Intel’s DL library for neural nets and NLP.\\[1ex]
10. \textbf{CNTK‑GPU}: GPU‑accelerated version of CNTK.\\[1ex]
\textbf{When choosing a library, consider:}\\
\quad * Ease of use (beginners: Keras, PyTorch).\\
\quad * Performance (large datasets/models: TensorFlow, CNTK).\\
\quad * Specific use case (CV: OpenCV; NLP: NLTK, spaCy).
\par\noindent\hrulefill\par
\textless{}\texttt{Llama-8B-Instruct-DPO$_{\textnormal{clean}}$-ComPO}\textgreater{} \\

Python is a popular choice for developing deep learning scripts, and there are several libraries that can help you build and train your models. Here are some of the most popular and widely-used libraries for deep learning in Python:\\[0.5ex]
1. \textbf{TensorFlow}: Open‑source library for numerical computation, fine‑tuned for large‑scale ML/DL tasks (Google Brain).\\
\quad \textbf{Pros:} Large community, widely used, scalable, flexible.\\
\quad \textbf{Cons:} Steeper learning curve, complex architecture.\\[1ex]
2. \textbf{PyTorch}: Open‑source ML library (Facebook FAIR), known for simplicity and dynamic graphs.\\
\quad \textbf{Pros:} Easy to learn, flexible, great for rapid prototyping.\\
\quad \textbf{Cons:} Smaller community than TensorFlow, less optimized for very large projects.\\[1ex]
3. \textbf{Keras}: High‑level neural nets API in Python (runs on TensorFlow, CNTK, Theano).\\
\quad \textbf{Pros:} User‑friendly, high‑level, multi‑backend.\\
\quad \textbf{Cons:} Limited low‑level control, can be slower at scale.\\[1ex]
4. \textbf{CNTK}: Microsoft’s Computational Network Toolkit.\\
\quad \textbf{Pros:} High‑performance, scalable, distributed training.\\
\quad \textbf{Cons:} Smaller community, less widespread.\\[1ex]
5. \textbf{Theano}: Python library for defining and optimizing mathematical expressions.\\
\quad \textbf{Pros:} Fast, flexible, research‑friendly.\\
\quad \textbf{Cons:} Steep learning curve, less user‑friendly.\\[1ex]
6. \textbf{MXNet}: Apache’s scalable DL library, multi‑language support.\\
\quad \textbf{Pros:} Lightweight, high‑performance, scalable.\\
\quad \textbf{Cons:} Smaller community, less mature.\\[1ex]
7. \textbf{Scikit‑Learn}: General ML library for classification, regression, clustering.\\
\quad \textbf{Pros:} Simple, excellent docs, broad algorithm set.\\
\quad \textbf{Cons:} Not DL‑focused, less powerful for deep networks.\\[1ex]
\textbf{When choosing a library, consider:}\\
\quad * Your experience level (beginners: Keras, TensorFlow).\\
\quad * Project scale (large: TensorFlow, PyTorch).\\
\quad * Project type (research: Theano, CNTK).\\
\quad * Level of control required (TensorFlow, PyTorch).\\[1ex]
Ultimately, the best library depends on your needs, goals, and preferences. \\
}
\end{tcolorbox}

\begin{tcolorbox}[
    title=Example 3: Mathematics,
    colback=white!97!gray,
    colframe=black!50,
    fonttitle=\bfseries,breakable]
    
{
\textless{}\texttt{User Prompt}\textgreater{} \\

Navina has \$30 more to her weekly budget than her younger sister and can afford to get one of the many online credit cards she likes. What do they each have to spend?
\par\noindent\hrulefill\par
\textless{}\texttt{Gemma-9B-Instruct-SimPO}\textgreater{} 
 \\

Here’s how to solve this problem:\\[0.8ex]
\textbf{Let’s use variables:}\\[0.5ex]
\quad * \textit{Sister’s weekly budget:}~x\\
\quad * \textit{Navina’s weekly budget:}~x + \$30\\[1ex]
\textbf{Since we don’t have a specific total budget to divide, we can only express their individual amounts:}\\[0.5ex]
\quad * \textit{Sister:}~x dollars per week\\
\quad * \textit{Navina:}~x + \$30 dollars per week.

\par\noindent\hrulefill\par
\textless{}\texttt{Gemma-9B-Instruct-SimPO-ComPO}\textgreater{} \\

Here’s how to solve this problem:\\[0.8ex]
\textbf{Let’s use variables:}\\[0.5ex]
\quad * Let ‘x’ represent the amount Navina’s sister has in her weekly budget.\\
\quad * Navina has \$30 more than her sister, so Navina has ‘x + \$30’ in her budget.\\[1ex]
\textbf{Answer:}\\[0.5ex]
\quad * \textit{Sister’s weekly budget:}~x\\
\quad * \textit{Navina’s weekly budget:}~x + \$30\\[1ex]
We need a specific number for ‘x’ to get exact amounts, but this setup shows the relationship between their budgets.\\
}
\end{tcolorbox}

\end{document}